\documentclass[pdflatex,sn-apa]{sn-jnl}

\usepackage{graphicx}%
\usepackage{subcaption}
\usepackage{multirow}%
\usepackage{amsmath,amssymb,amsfonts}%
\usepackage{amsthm}%
\usepackage{mathrsfs}%
\usepackage[title]{appendix}%
\usepackage{xcolor}%
\usepackage{textcomp}%
\usepackage{manyfoot}%
\usepackage{booktabs}%
\usepackage{algorithm}%
\usepackage{algorithmicx}%
\usepackage{algpseudocode}%
\usepackage{listings}%
\usepackage{rotating}%

\newcommand{\R}{\mathbb{R}}
\newcommand{\Pa}{\mathrm{Pa}}
\newcommand{\Ch}{\mathrm{Ch}}
\newcommand{\An}{\mathrm{An}}
\newcommand{\De}{\mathrm{De}}
\newcommand{\NonDesc}{\mathrm{NonDesc}}
\newcommand{\Var}{\mathrm{Var}}
\newcommand{\Cov}{\mathrm{Cov}}
\newcommand{\E}{\mathbb{E}}

\theoremstyle{thmstyleone}%
\newtheorem{theorem}{Theorem}
\newtheorem{proposition}[theorem]{Proposition}%
\newtheorem{lemma}[theorem]{Lemma}%
\newtheorem{assumption}[theorem]{Assumption}%
\newtheorem{corollary}[theorem]{Corollary}%

\theoremstyle{thmstyletwo}%
\newtheorem{example}{Example}%
\newtheorem{remark}{Remark}%

\theoremstyle{thmstylethree}%
\newtheorem{definition}{Definition}%

\newlength{\expfigheight}
\begin{document}

\title{AdaPS-LiNGAM: Adaptive Predecessor Selection for Linear Non-Gaussian Acyclic Models under Small-Sample Settings}








\author*[1]{\fnm{Shun} \sur{Yanashima}}\email{yanashima-shun@fujitsu.com}

\author[1]{\fnm{Kentaro} \sur{Kanamori}}\email{k.kanamori@fujitsu.com}

\author[1]{\fnm{Hirofumi} \sur{Suzuki}}\email{suzuki-hirofumi@fujitsu.com}

\affil*[1]{\orgdiv{Artificial Intelligence Laboratory}, \orgname{Fujitsu Limited}, \orgaddress{\street{4-1-1 Kamikodanaka, Nakahara-ku}, \city{Kawasaki}, \postcode{211-8588}, \state{Kanagawa}, \country{Japan}}}


\abstract{
Causal discovery becomes particularly challenging when the available sample size is small relative to the number of variables. 
This challenge also arises in the linear non-Gaussian acyclic model (LiNGAM), an identifiable framework for causal discovery from observational data. 
DirectLiNGAM estimates a causal order, which arranges variables so that causes precede their effects, by sequentially identifying an exogenous variable and removing its linear effect from the remaining variables. 
We establish a structural limitation of this procedure: when the number of variables exceeds the sample size, repeated residualization necessarily becomes degenerate before the full causal order can be determined. 
Our analysis further reveals that each residual can be reconstructed using only a graph-determined subset of variables already placed earlier in the causal order, termed the \emph{active boundary}.
This result motivates AdaPS-LiNGAM (Adaptive Predecessor Selection LiNGAM), which reconstructs each residual directly from the original observations using an adaptively chosen sparse subset of those earlier variables.
The same subset-selection principle is also applied to the final pruning step for edge estimation. 
Experiments on synthetic data demonstrate that AdaPS-LiNGAM provides accurate causal-structure recovery in sample-limited settings and degrades more gradually as the sample size decreases.
}

\keywords{Causal discovery, LiNGAM, Sparse regression, Best-subset selection}



\maketitle

\section{Introduction}\label{sec_intro}

Causal discovery aims to infer causal relationships among variables from data and has been applied across a wide range of scientific and practical domains \citep{sachs2005causal,maathuis2010predicting,moneta2013causal,zhang2013integrated,runge2019inferring}. 
A central theoretical issue is that causal direction generally cannot be identified from observational data alone without additional assumptions \citep{spirtes2000causation}. 
The linear non-Gaussian acyclic model (LiNGAM) addresses this issue by exploiting the non-Gaussianity of disturbance variables to identify the causal structure from observational data under a linear acyclic model \citep{shimizu2006lingam}. This has led to a number of practical methods for estimating directed acyclic graphs within the LiNGAM framework.

A fundamental challenge for causal discovery in practice is that the available sample size is often limited. 
In this paper, we focus on settings in which the number of observations $n$ is small in absolute terms and the number of variables $p$ is comparable to or larger than $n$. 
Depending on the method, causal discovery in this regime may require either estimating regression models with many candidate predictors or assessing conditional independence using large conditioning sets, even though only limited data are available. 
Related problems have been studied extensively in small-sample and high-dimensional causal discovery. 
For example, previous studies have considered the reliability of causal discovery under limited data \citep{bromberg2009improving}, the use of lower-order independence information \citep{textor2015learning,wienobst2020recovering}, and approaches that reduce the effective size of the structure-learning problem \citep{cai2013sada,xie2019mislingam}. 
Taken together, these studies underscore the importance of controlling the dimension of each estimation or testing problem when data are scarce.

Within the LiNGAM framework, this concern is directly relevant because several widely used algorithms rely explicitly on regression residuals. 
DirectLiNGAM \citep{shimizu2011directlingam}, in particular, proceeds sequentially by identifying an exogenous variable and removing its linear effect from the remaining variables. 
At each step, the residuals of the remaining variables are obtained by regressing them on variables that have already been selected. 
Thus, as the procedure proceeds, the number of selected variables involved in the residualization grows cumulatively. 
This construction is central to the DirectLiNGAM procedure, but it also introduces a potential difficulty in finite samples: the regression underlying each residual can involve an increasingly large set of predictors. More generally, regression-based estimation with many predictors can become statistically unstable when the sample size is limited. High-dimensional LiNGAM methods have therefore exploited sparsity or otherwise restricted the set of variables used for estimation \citep{sogawa2011estimating,wang2020high}.

In this work, we show that cumulative residualization creates a problem more fundamental than increased estimation error. DirectLiNGAM performs each residualization as an orthogonal projection in the finite-dimensional sample space. Each nonzero step therefore reduces by one the dimension of the space in which the residual vectors lie. Because centering $n$ observations leaves at most $n-1$ dimensions available, the residual vectors must become zero after at most $n-1$ such reductions. Consequently, when $n<p$, this occurs before DirectLiNGAM has placed every variable in a sequence in which causes precede their effects. Thus, in the small-sample regime, the residualization procedure itself can become structurally degenerate, rather than merely lose statistical accuracy.

This observation raises a natural question: is it necessary to use all previously selected variables when constructing each residual? 
Our analysis shows that, provided that the variables selected so far follow an order consistent with the underlying causal structure, the residual of each remaining variable can be reconstructed using only a subset of them.
We formalize this subset through a graph-theoretic notion that we call the \emph{active boundary}. 
The active boundary consists of previously selected variables whose effects can enter the remaining part of the graph and subsequently reach the target variable without passing through another selected variable. 
Under a localized non-cancellation condition, we show that this set coincides with the support of the corresponding population reduced-form regression coefficients. Consequently, regression on the active boundary alone yields exactly the same residual as the cumulative DirectLiNGAM residualization and therefore preserves the corresponding DirectLiNGAM ordering step.
This characterization also clarifies why sparse residual reconstruction can be useful in small-sample settings. The number of variables required to reconstruct a residual is determined not by the total number of previously selected variables, but by the size of the active boundary induced by the causal graph and the current ordering step. Hence, even when the total number of variables exceeds the sample size, the relevant residual reconstruction can remain low-dimensional when the active boundaries are sufficiently sparse. This observation provides a structural route to avoiding the finite-sample degeneracy of cumulative residualization.

Based on this idea, we propose \emph{AdaPS-LiNGAM} (Adaptive Predecessor Selection LiNGAM), a method built on the DirectLiNGAM framework that adaptively restricts the predictors used for residual reconstruction. 
At each ordering step, AdaPS-LiNGAM returns to the original data and reconstructs the residual of each remaining variable using only a subset of the previously selected variables, rather than repeatedly residualizing the current residuals against all of them.
Because the target subset is unknown from finite data, we estimate it using adaptive best-subset selection together with an information criterion. 
Specifically, we use the ABESS algorithm \citep{zhu2020abess} to efficiently solve the best-subset selection problem at each ordering step. 
The same principle is also used in the final pruning stage to avoid unnecessary predictors.

Our contributions are as follows:
\begin{enumerate}
\item We characterize a previously unrecognized finite-sample degeneracy of the cumulative residualization procedure in DirectLiNGAM and establish that, when $n<p$, the residual vectors necessarily vanish while multiple variables can still remain unordered.
\item We introduce the active boundary and characterize it as the unique minimal predictor set required to reconstruct the DirectLiNGAM residual under a localized non-cancellation condition. 
\item We develop AdaPS-LiNGAM, an adaptive predecessor-selection procedure that operationalizes sparse residual reconstruction in finite samples, and evaluate the proposed method from two complementary perspectives: robustness to increasing dimensionality in the small-sample regime and degradation as the number of observations decreases.
\end{enumerate}

The remainder of this paper is organized as follows. 
Section~\ref{sec:related_work} reviews related work on small-sample and high-dimensional causal discovery, with particular emphasis on LiNGAM. 
Section~\ref{sec:prelim} introduces the LiNGAM model and the DirectLiNGAM procedure. 
Section~\ref{sec:limitation} analyzes the finite-sample degeneracy caused by cumulative residualization. 
Section~\ref{sec:method} presents AdaPS-LiNGAM, including a characterization of the active boundary for residual reconstruction, its relation to sample size, and the adaptive subset-selection procedure. 
Section~\ref{sec:experiments} reports the empirical evaluation. 
Finally, Section~\ref{sec:conclusion} concludes the paper.

\section{Related Work}
\label{sec:related_work}

\paragraph{Causal Discovery in Small-Sample Settings.}
Causal discovery from observational data can be challenging when only a limited number of observations is available. In this regime, statistical quantities used by causal discovery methods can become unreliable, particularly when they involve a large number of variables. This issue is well recognized in constraint-based methods, where conditional independence tests with large conditioning sets can have low statistical power in finite samples. \citet{bromberg2009improving} studied the reliability of causal discovery from small data sets and proposed an argumentation-based approach for dealing with inconsistencies among estimated independence relations. Several lines of work have therefore considered causal discovery using lower-order independence information. \citet{textor2015learning} studied the causal information that can be obtained from pairwise marginal independencies, and \citet{wienobst2020recovering} generalized this to conditional independencies with conditioning sets of size at most $k$. \citet{zhang2017causal} proposed regression-based conditional independence tests that, under an additive noise model, reformulate conditional independence testing in terms of independence between regression residuals. \citet{kocaoglu2023small} characterized the information that can be learned from conditional independence relations whose conditioning sets have size at most $k$ through the notion of $k$-Markov equivalence, and proposed the sound $k$-PC algorithm for this setting. Another line of work reduces the size of the causal discovery problem itself. \citet{cai2013sada} proposed SADA, a split-and-merge framework that exploits local sparsity to divide causal discovery into smaller subproblems, and demonstrated its applicability when the sample size is substantially smaller than the number of variables. These approaches differ in their assumptions and objectives, but share the broader observation that limiting the amount of statistical information used at each stage can be beneficial when data are scarce. We consider this issue for regression-based causal discovery, where the relevant statistical operation is the construction of regression residuals.

\paragraph{LiNGAM in Small-Sample Settings.}
LiNGAM provides a particularly relevant setting for studying this problem. \citet{shimizu2006lingam} established the identifiability of linear non-Gaussian acyclic models from observational data by exploiting the non-Gaussianity of the disturbance variables. DirectLiNGAM subsequently introduced a direct procedure for estimating the causal ordering by identifying exogenous variables from independence relationships between variables and regression residuals \citep{shimizu2011directlingam}. \citet{hyvarinen2013pairwise} proposed pairwise likelihood ratios as an alternative criterion for estimating non-Gaussian structural equation models and discussed their applicability when the number of variables exceeds the number of observations, as well as their behavior with relatively few observations. DirectLiNGAM is particularly relevant to the present work because, after a partial causal ordering has been obtained, its procedure constructs a residual for a remaining variable using variables that precede it in the ordering. When the sample size is limited, estimating this regression with a large predictor set can become difficult. \citet{xie2019mislingam} addressed small-sample LiNGAM more directly with MiS-LiNGAM, which first estimates a skeleton and then identifies minimal LiNGAM sets so that causal discovery can be performed on smaller subsets of variables. Thus, MiS-LiNGAM reduces the size of the causal-discovery problem itself, whereas our approach retains the DirectLiNGAM framework and reduces the predictor set used to construct each regression residual.

\paragraph{Causal Discovery in High-Dimensional Settings.}
A related but distinct body of work considers the high-dimensional setting, in which the number of variables is large relative to the number of observations. \citet{sogawa2011estimating} considered the estimation of exogenous variables in LiNGAM when the number of variables exceeds the number of observations, using a non-Gaussianity-based criterion. \citet{loh2014high} proposed a framework for high-dimensional linear SEMs that first estimates a sparse moralized graph and then restricts the subsequent search for a DAG to structures compatible with that graph. \citet{wang2020high} developed a high-dimensional causal discovery method for linear non-Gaussian models and established consistency under sparsity conditions involving the maximum in-degree. \citet{park2021learning} studied high-dimensional linear SEMs using $\ell_1$-regularized regression for causal-order and parent-set estimation, without relying on faithfulness, Gaussian errors, or equal error variances. These methods primarily address the difficulty caused by the growth of the ambient dimension relative to the sample size by exploiting structural sparsity or restricting the space of candidate models. The high-dimensional setting overlaps with the small-sample regime when $p$ is large relative to $n$, but the two notions are not identical: limited sample size can also be problematic when $n$ exceeds $p$. Our focus is on the latter finite-sample issue and, in particular, on how the number of predictors used in the regression underlying each causal-ordering step affects estimation.



\section{Preliminaries}\label{sec:prelim}

We use the following notation throughout the paper. 
Let $V:=\{1,\dots,p\}$ denote the variable index set. 
For a vector $a=(a_1,\dots,a_p)^\top\in \R^p$ and an index set $S\subseteq V$, $a_S$ denotes the subvector in $\R^{|S|}$ indexed by $S$. 
For a matrix $M\in\R^{p\times p}$ and index sets $S,T\subseteq V$, $M_{ST}$ denotes the submatrix in $\R^{|S|\times|T|}$ with rows indexed by $S$ and columns indexed by $T$.

\subsection{LiNGAM Model}\label{subsec:lingam}

Let $x=(x_1,\dots,x_p)^\top$ be the observed random vector. The linear non-Gaussian acyclic model (LiNGAM) \citep{shimizu2006lingam} assumes the linear structural equations
\begin{equation}\label{eq:lingam}
 x_i=\sum_{j=1}^{p} b_{ij}x_j+e_i,
 \qquad i=1,\dots,p,
\end{equation}
where $b_{ij}\in\R$ represents the direct effect from $x_j$ to $x_i$, and $e_i$ is an exogenous noise variable. 
The noise variables $e_1,\dots,e_p$ are mutually independent, continuous random variables with non-Gaussian distributions and non-zero variances. 
We assume that the observed variables have been centered, so that $\E[x_i]=0$ for all $i$. This entails no loss of generality for the linear model.

Let $B=(b_{ij})\in\R^{p\times p}$ and $e=(e_1,\dots,e_p)^\top$. 
Then the LiNGAM model in~\eqref{eq:lingam} can be written in vector form as
\begin{equation*}
 x=Bx+e.
\end{equation*}
The coefficient matrix $B$ defines a directed graph $G=(V,E)$ by $j\to i$ if and only if $b_{ij}\neq0$. LiNGAM assumes that $G$ is a directed acyclic graph (DAG).

For distinct nodes $i,j\in V$, write $j\leadsto i$ when there exists a directed path from $j$ to $i$, and define
\begin{align*}
 \Pa_i &:= \{j\in V\setminus\{i\}: b_{ij}\neq0\},\\
 \Ch_i &:= \{j\in V\setminus\{i\}: b_{ji}\neq0\},\\
 \An_i &:= \{j\in V\setminus\{i\}: j\leadsto i\},\\
 \De_i &:= \{j\in V\setminus\{i\}: i\leadsto j\},\\
 \NonDesc_i &:= V\setminus\bigl(\De_i\cup\{i\}\bigr).
\end{align*}
Thus $\Pa_i$ is the parent set, $\Ch_i$ is the child set, $\An_i$ is the strict ancestor set, $\De_i$ is the strict descendant set, and $\NonDesc_i$ is the set of strict non-descendants. Equivalently, $j\in\Pa_i$ if and only if $j\to i$, and $j\in\Ch_i$ if and only if $i\to j$. In a DAG, $\Pa_i\subseteq\An_i\subseteq\NonDesc_i$.

A causal order of the variables is defined by a permutation $\pi:V\to V$ such that
\[
b_{ij}\neq0
 \quad\Longrightarrow\quad
 \pi(j)<\pi(i).
\]
Thus $\pi(i)$ is the position of variable $x_i$ in the causal order. 
Equivalently, simultaneously permuting the rows and columns of $B$ according to increasing $\pi(i)$ yields a strictly lower-triangular matrix. Because $G$ is acyclic, at least one causal order exists, although it need not be unique.

A fundamental result of LiNGAM is that, under the above model assumptions, the causal structure encoded by $B$ is identifiable from the observed distribution \citep{shimizu2006lingam}. This identifiability rests on the following Darmois--Skitovich theorem.
\begin{theorem}[Darmois--Skitovich]\label{thm:ds}
Let $s_1,\dots,s_m$ be independent random variables, and consider two linear combinations $u=\sum_{k=1}^m\alpha_k s_k$ and $v=\sum_{k=1}^m\gamma_k s_k$. If $u$ and $v$ are independent, then for every $k$ with $\alpha_k\gamma_k\neq0$, the variable $s_k$ is Gaussian.
\end{theorem}

The contrapositive states that if a non-Gaussian independent component appears in both linear combinations with nonzero coefficients, then the two combinations cannot be independent. LiNGAM exploits this to identify the causal direction. 
We illustrate this with the simplest two-variable example. 
Suppose the true structure is $x_1\to x_2$ with $x_1=e_1$ and $x_2=b_{21}x_1+e_2$. 
In the correct direction, regressing $x_2$ on $x_1$ yields the residual $e_2$, which is independent of $x_1\,(=e_1)$. 
On the other hand, in the wrong direction, regressing $x_1$ on $x_2$ yields a residual that is a linear combination containing both $e_1$ and $e_2$ with nonzero coefficients. 
Since it shares a non-Gaussian component with $x_2$, (the contrapositive of) the Darmois--Skitovich theorem implies that they cannot be independent. 
This asymmetry makes the causal direction identifiable under non-Gaussianity.

\subsection{DirectLiNGAM}\label{subsec:directlingam}

DirectLiNGAM \citep{shimizu2011directlingam} estimates a causal order sequentially by identifying an exogenous variable among the remaining variables and removing its linear effect from the others. 
At the population level, the identification principle can be expressed through pairwise regression residuals. 
For two variables $u$ and $v$ under consideration, define the linear-regression residual of $u$ on $v$ by
\[
r_{u\mid v}
 :=u-\frac{\Cov(u,v)}{\Var(v)}v.
\]
Under the LiNGAM assumptions in Section~\ref{subsec:lingam}, a variable is exogenous if and only if it is independent of the residual obtained by regressing every other variable under consideration on it. 

Let
\[
 X=(X_1,\dots,X_p)\in\R^{n\times p}
\]
be the centered data matrix, where $X_i\in\R^n$ is the observation vector for $x_i$. 
At iteration $t$ of the sample-level ordering procedure of DirectLiNGAM, let $\widehat U_t\subseteq V$ denote the indices already placed in the first $t$ positions and let
\[
\widehat V_t:=V\setminus \widehat U_t,
 \qquad
 \widehat U_0=\varnothing,
 \qquad
 \widehat V_0=V.
\]
For each $i\in\widehat V_t$, write $R_i^{(t)}\in\R^n$ for the current residual vector, with $R_i^{(0)}:=X_i$, and collect them in the residual data matrix
\[
 R^{(t)}:=(R_i^{(t)})_{i\in\widehat V_t}\in\R^{n\times|\widehat V_t|}.
\]

For a sample matrix $Y=(Y_i)_{i\in S}\in\R^{n\times|S|}$ and distinct $i,j\in S$, define the sample pairwise residual vector
\begin{equation}\label{eq:direct-pairwise-residual}
 R_{i\mid j}(Y)
 :=Y_i-\frac{Y_j^\top Y_i}{\|Y_j\|_2^2}Y_j,
\end{equation}
whenever $Y_j\neq\mathbf 0$. 
Let $\mathcal D_n(a,b)$ denote the sample dependence measure used to operationalize the independence criterion, with smaller values indicating weaker dependence. 
The concrete form of $\mathcal D_n$ depends on the DirectLiNGAM implementation. 
Define
\begin{equation}\label{eq:direct-root-score}
 T_n(j;Y)
 :=\sum_{i\in S\setminus\{j\}}
 \mathcal D_n\!\left(Y_j,R_{i\mid j}(Y)\right),
 \qquad j\in S,
\end{equation}
for $Y=(Y_i)_{i\in S}$. 
DirectLiNGAM then chooses
\[
\widehat j_t\in\arg\min_{j\in\widehat V_t} T_n(j;R^{(t)}).
\]
After selecting $\widehat j_t$, DirectLiNGAM updates every remaining residual vector by
\begin{equation}\label{eq:direct-sequential-update}
 R_i^{(t+1)}
 :=R_{i\mid \widehat j_t}(R^{(t)}),
 \qquad i\in\widehat V_t\setminus\{\widehat j_t\},
\end{equation}
and sets
\begin{equation*}
 \widehat U_{t+1}:=\widehat U_t\cup\{\widehat j_t\},
 \qquad
 \widehat V_{t+1}:=\widehat V_t\setminus\{\widehat j_t\}.
\end{equation*}
Hence, after $t$ variables have been selected, every remaining residual vector has undergone $t$ successive least-squares projections. This cumulative residualization is analyzed in Section~\ref{sec:limitation}.

Once all variables have been ordered, DirectLiNGAM estimates the coefficient matrix $B$ by regressing each variable on variables preceding it in the estimated causal order and removing unnecessary edges. 
We refer to this final edge-selection step as \emph{pruning}. 
Adaptive Lasso \citep{zou2006adaptivelasso} and related procedures are commonly used for this purpose. Algorithm~\ref{alg:directlingam} summarizes the sample-level procedure.

\begin{algorithm}[t]
\caption{DirectLiNGAM}\label{alg:directlingam}
\begin{algorithmic}[1]
\Require Centered data matrix $X=(X_1,\dots,X_p)\in\R^{n\times p}$
\State $\widehat U_0\gets\varnothing$, $\widehat V_0\gets V$, and $R_i^{(0)}\gets X_i$ for all $i\in V$
\For{$t=0,\dots,p-2$}
  \State $R^{(t)}\gets (R_i^{(t)})_{i\in\widehat V_t}$
  \State choose $\widehat j_t\in\arg\min_{j\in\widehat V_t} T_n(j;R^{(t)})$
  \State $\widehat\pi(\widehat j_t)\gets t+1$
  \For{each $i\in\widehat V_t\setminus\{\widehat j_t\}$}
    \State $R_i^{(t+1)}\gets R_{i\mid \widehat j_t}(R^{(t)})$
  \EndFor
  \State $\widehat U_{t+1}\gets \widehat U_t\cup\{\widehat j_t\}$ and $\widehat V_{t+1}\gets \widehat V_t\setminus\{\widehat j_t\}$
\EndFor
\State assign $\widehat\pi(j)\gets p$ for the unique $j\in\widehat V_{p-1}$
\State estimate $\widehat B$ by regressing each variable on its predecessors under $\widehat\pi$ and pruning unnecessary edges
\State \Return $\widehat\pi,\widehat B$
\end{algorithmic}
\end{algorithm}


\section{Limitations of DirectLiNGAM in Small-Sample Settings}\label{sec:limitation}

This section establishes a finite-sample limitation of DirectLiNGAM: when $n<p$, its cumulative residualization becomes degenerate before the full causal order can be determined. 
The DirectLiNGAM update in~\eqref{eq:direct-sequential-update} is an orthogonal projection in the finite-dimensional sample space. 
In Section~\ref{subsec:subspace-view}, we first show that each residualization step removes one dimension from the subspace spanned by the residuals whenever the selected residual is nonzero, so this dimension reduction can occur at most $\operatorname{rank}(X)\le n-1$ times. 
In Section~\ref{subsec:residual-degeneracy}, we then establish that, when $n<p$, this rank constraint forces the remaining residual vectors to vanish while at least two variables remain unordered. 
This limitation follows from the finite-sample geometry of the residual update itself.

\subsection{Projection Geometry of Sequential Residualization}\label{subsec:subspace-view}

We first analyze the projection geometry of sequential residualization and show that each projection reduces the dimension of the subspace spanned by the residuals by one whenever the selected residual is nonzero.

Let $\mathbf 1_n\in\R^n$ denote the all-ones vector and define the centered sample subspace
\[
\mathcal H_0:=\{v\in\R^n:\mathbf 1_n^\top v=0\},
 \qquad \dim(\mathcal H_0)=n-1.
\]
Because every column of $X$ is centered, $X_i\in\mathcal H_0$ for all $i\in V$. Let
\begin{equation}\label{eq:data-subspace}
 \mathcal S_0:=\operatorname{span}\{X_i:i\in V\},
 \qquad r:=\dim(\mathcal S_0)=\operatorname{rank}(X).
\end{equation}
Since $\mathcal S_0\subseteq\mathcal H_0$, we have
\[
r\le n-1.
\]
Let $I_n$ denote the $n\times n$ identity matrix. For a nonzero vector $z\in\R^n$, define $z^\perp:=\{v\in\R^n:z^\top v=0\}$ and the orthogonal projector onto $z^\perp$,
\[
P_{z^\perp}:=I_n-\frac{zz^\top}{z^\top z}.
\]
If $\widehat j_t$ is selected at iteration $t$ of DirectLiNGAM, define $g_t:=R_{\widehat j_t}^{(t)}$.
Then~\eqref{eq:direct-sequential-update} can be written as
\begin{equation}\label{eq:projection-update}
 R_i^{(t+1)}=P_{g_t^\perp} R_i^{(t)},
 \qquad i\in\widehat V_{t+1},
\end{equation}
which is defined only when $g_t\neq\mathbf 0$. 
For $t\ge1$, define
\[
\mathcal S_t:=\mathcal S_0\cap g_0^\perp\cap\cdots\cap g_{t-1}^\perp.
\]

With these definitions, the following lemma shows that, provided that $g_0,\dots,g_{t-1}$ are nonzero, the updated residuals span the nested subspaces $\mathcal S_t$, whose dimension decreases by one at each step.

\begin{lemma}[Nested residual subspaces]\label{lem:nested-residual-subspaces}
Suppose $g_0,\dots,g_{t-1}$ are nonzero. Then
\[
\operatorname{span}\{R_i^{(t)}:i\in\widehat V_t\}=\mathcal S_t,
 \qquad
 \dim(\mathcal S_t)=r-t.
\]
\end{lemma}
\begin{proof}
For $t=0$ the claim follows from~\eqref{eq:data-subspace}, since $R_i^{(0)}=X_i$ and $V=\widehat V_0$. 

Suppose $\operatorname{span}\{R_i^{(t)}:i\in\widehat V_t\}=\mathcal S_t$. 
The selected direction $g_t$ is one of these residual vectors, so $g_t\in\mathcal S_t$, and it is nonzero by hypothesis. 
Hence $\mathcal S_{t+1}=\mathcal S_t\cap g_t^\perp$ has dimension $\dim(\mathcal S_t)-1$. 
The restriction of $P_{g_t^\perp}$ to $\mathcal S_t$ has kernel $\operatorname{span}\{g_t\}$, so its image has dimension $\dim(\mathcal S_t)-1$. 
This image is contained in $\mathcal S_{t+1}$, and the two spaces are therefore equal. 
By~\eqref{eq:projection-update} the image is spanned by the updated residuals, and $P_{g_t^\perp} g_t=\mathbf 0$, so removing the selected index leaves the span unchanged.
Thus $\operatorname{span}\{R_i^{(t+1)}:i\in\widehat V_{t+1}\}=\mathcal S_{t+1}$.
Induction from $\dim(\mathcal S_0)=r$ gives $\dim(\mathcal S_t)=r-t$.
\end{proof}

\subsection{Vanishing Residuals under Cumulative Residualization}\label{subsec:residual-degeneracy}

We next apply Lemma~\ref{lem:nested-residual-subspaces} to characterize the finite-sample degeneracy caused by cumulative residualization when $n<p$.

\begin{proposition}[Vanishing residuals under cumulative residualization]
\label{prop:vanishing-residuals}
Suppose $n<p$, let $r=\operatorname{rank}(X)$, and suppose $g_0,\dots,g_{r-1}$ are nonzero. 
Then, for $t=0,\dots,r$,
\[
 \operatorname{rank}R^{(t)}=r-t,
 \qquad
 |\widehat V_t|=p-t.
\]
At $t=r$,
\[
 R_i^{(r)}=\mathbf 0
 \quad\text{for all }i\in\widehat V_r,
 \qquad
 |\widehat V_r|=p-r\ge p-n+1\ge2.
\]
Thus the remaining residual vectors vanish no later than step $n-1$ while at least two variables remain unordered.
\end{proposition}
\begin{proof}
By Lemma~\ref{lem:nested-residual-subspaces},
\[
 \operatorname{rank}R^{(t)}=\dim(\mathcal S_t)=r-t.
\]
Since one index is removed at each step, $|\widehat V_t|=p-t$. 
At $t=r$, $\mathcal S_r=\{\mathbf 0\}$, so every remaining residual vector vanishes.
Finally, $r\le n-1$ gives
\[
 p-r\ge p-n+1\ge2.
\]
\end{proof}

Proposition~\ref{prop:vanishing-residuals} implies that, when $n<p$,
DirectLiNGAM cannot complete causal-order estimation through cumulative residualization: either a zero selected residual is encountered before step $r$, or all remaining residuals vanish at step $r$ while at least two variables remain unordered.
In the latter case, under~\eqref{eq:direct-pairwise-residual}, $R_{i\mid j}(R^{(r)})$ is undefined for every remaining candidate $j$ because $R_j^{(r)}=\mathbf 0$. 
Consequently, the next DirectLiNGAM score in~\eqref{eq:direct-root-score} is not defined.

The relevant distinction is whether $r<p-1$. 
If $r=p-1$, simultaneous vanishing occurs only at $t=p-1$, when a single variable remains and no further ordering decision is required. 
If $r<p-1$, all remaining residuals vanish while at least two variables remain unordered. 
Since $n<p$ implies $r\le n-1<p-1$, the small-sample regime necessarily falls into the latter case.

When $n<p$ and the centered data contain no additional exact linear dependencies, in the sense that every subset of at most $n-1$ columns of $X$ is linearly independent, we have $r=n-1$. Under the same condition, no remaining residual vanishes before $t=r$. Thus the remaining residuals vanish simultaneously at step $n-1$.


\section{Proposed Method: AdaPS-LiNGAM}\label{sec:method}
The analysis in Section~\ref{sec:limitation} traces the failure of DirectLiNGAM in small-sample settings to its sequential residual updating. 
To address this issue, we propose Adaptive Predecessor Selection LiNGAM (AdaPS-LiNGAM). 
It avoids the cumulative residualization by returning to the original data at every ordering step and reconstructing each residual using only a subset of the previously selected variables.

First, in Section~\ref{subsec:theory}, we introduce the graph-theoretic notion of the \emph{active boundary}, which captures the previously selected nodes whose effects can still reach a remaining node without passing through another selected node. We then show that, under a localized non-cancellation condition, this graph-defined set is exactly the unique minimal predictor set required for exact residual reconstruction. 
Next, in Section~\ref{subsec:boundary-sparsity}, we study how large the active boundaries can become across ordering steps and establish a sample-size condition under which residual reconstruction using them remains feasible. 
Finally, in Section~\ref{subsec:algo}, we show how the predictor sets characterized by the active boundaries can be estimated from finite samples using adaptive best-subset selection, leading to the AdaPS-LiNGAM algorithm.

\subsection{Characterization of Sparse Residual Reconstruction}
\label{subsec:theory}
This subsection characterizes the predictor set required for residual reconstruction from the original data. 
We define the \emph{active boundary} and show, under a non-cancellation condition, that it coincides with the support of the corresponding population regression coefficients.

Fix a causal order $\pi:V\to V$. 
For $t=0,\dots,p$, define
\[
U_t^\pi:=\{i\in V:\pi(i)\le t\},
\qquad
V_t^\pi:=V\setminus U_t^\pi.
\]
We refer to $U_t^\pi$ as the \emph{prefix set} and to $V_t^\pi$ as the \emph{remaining set}. 
When the causal order is fixed, we write $U_t$ and $V_t$ for $U_t^\pi$ and $V_t^\pi$, respectively. 
Because $U_t$ is a prefix of a causal order, it is ancestral: no node in $V_t$ is an ancestor of a node in $U_t$.

We first record an elementary projection fact, which allows the sequential residualization of Section~\ref{subsec:directlingam} to be replaced by a single joint regression on $x_{U_t}$.

For each $t$, define the reduced-form coefficient matrix
\[
 \Theta^{(t)}
 :=(\theta_{ij}^{(t)})_{i\in V_t,\;j\in U_t}
 :=(I-B_{V_tV_t})^{-1}B_{V_tU_t}.
\]

\begin{lemma}[Reduced regression on a prefix set]\label{lem:reduced-regression}
For a prefix set $U_t$, the successive DirectLiNGAM residuals satisfy
\begin{align}
 x_{V_t}
 &=\Theta^{(t)}x_{U_t}+r_{V_t}^{(t)}, \label{eq:reduced-regression}\\
r_{V_t}^{(t)}
 &=(I-B_{V_tV_t})^{-1}e_{V_t}. \label{eq:reduced-residual}
\end{align}
Thus $r_{V_t}^{(t)}$ is also the residual of the linear regression of $x_{V_t}$ on $x_{U_t}$.
\end{lemma}
\begin{proof}
Because $U_t$ is ancestral, $B_{U_tV_t}=0$. 
Partitioning the LiNGAM equations according to $(U_t,V_t)$ gives
\begin{align*}
 x_{U_t} &= B_{U_tU_t}x_{U_t}+e_{U_t},\\
 x_{V_t} &= B_{V_tU_t}x_{U_t}
          +B_{V_tV_t}x_{V_t}+e_{V_t}.
\end{align*}
Since $I-B_{V_tV_t}$ is invertible,
\[
 x_{V_t}
 =\Theta^{(t)}x_{U_t}
 +(I-B_{V_tV_t})^{-1}e_{V_t}.
\]
Since $x_{U_t}$ is a function only of $e_{U_t}$ and $e_{V_t}$ is independent of $e_{U_t}$,
\[
 \Cov\!\left((I-B_{V_tV_t})^{-1}e_{V_t},x_{U_t}\right)=0.
\]
Hence $\Theta^{(t)}x_{U_t}$ is the linear projection of $x_{V_t}$ onto $x_{U_t}$, with residual $(I-B_{V_tV_t})^{-1}e_{V_t}$.
Successive DirectLiNGAM residualization removes orthogonal directions spanning $\operatorname{span}(x_{U_t})$, and therefore yields the same orthogonal-projection residual. 
By uniqueness of the orthogonal projection,
\[
 r_{V_t}^{(t)}
 =(I-B_{V_tV_t})^{-1}e_{V_t},
\]
which proves~\eqref{eq:reduced-regression} and~\eqref{eq:reduced-residual}.
\end{proof}

Lemma~\ref{lem:reduced-regression} shows that, after $t$ ordering steps, the residual of DirectLiNGAM for target $i$ can be obtained from a single regression of $x_i$ on $x_{U_t}$, with coefficient row $\Theta^{(t)}_{i,\cdot}$. 
Thus, identifying which variables in $U_t$ are needed for residual reconstruction reduces to determining which entries of $\Theta^{(t)}_{i,\cdot}$ are nonzero.

\begin{definition}[Active boundary]\label{def:active-boundary}
For a prefix set $U_t$ and a remaining node $i\in V_t$, define
\[
 \partial_i(U_t)
 :=\left\{j\in U_t:\ \exists\ j=v_0\to v_1\to\cdots\to v_m=i,\ m\ge1,
 \text{ with }v_1,\ldots,v_m\in V_t\right\}.
\]
Thus $\partial_i(U_t)$ contains the nodes in the prefix set whose effects enter the remaining set and can subsequently reach $i$ without passing through another node in the prefix set.
\end{definition}

By the path expansion of the reduced-form coefficient matrix,
\begin{equation}\label{eq:neumann}
\Theta^{(t)}
  =
 \left(\sum_{\ell=0}^{|V_t|-1}B_{V_tV_t}^{\ell}\right)B_{V_tU_t}.
\end{equation}
Hence each coefficient $\theta_{ij}^{(t)}$ is the sum of products of edge coefficients over directed paths from $j$ to $i$ whose nodes after $j$ remain in $V_t$. 
Therefore,
\begin{equation}\label{eq:support-inclusion}
 \operatorname{supp}(\Theta^{(t)}_{i,\cdot})
 \subseteq \partial_i(U_t).
\end{equation}
The inclusion can be strict if the contributions of multiple directed paths cancel exactly. 
We exclude this case by the following assumption.

\begin{assumption}[Non-cancellation condition]
\label{ass:non-cancellation}
For every ordering step $t$, every $i\in V_t$, and every $j\in\partial_i(U_t)$,
\[
 \theta_{ij}^{(t)}\neq0.
\]
\end{assumption}

This assumption is related to, but distinct from, the classical faithfulness condition \citep{spirtes2000causation}. 
In linear Gaussian structural equation models, violations of faithfulness can arise from exact cancellations among the contributions of multiple directed paths \citep{peters2012}. 
Such cancellations correspond to nontrivial algebraic constraints on the causal parameters and hence occur on a measure-zero subset of the parameter space under the usual absolute-continuity argument \citep{meek1995faithfulness,uhler2013geometry}.

Our assumption imposes only the localized non-cancellation condition needed for the reduced-form coefficients in~\eqref{eq:reduced-regression}: for each active-boundary pair $(j,i)$, the relevant directed-path contributions to $\theta_{ij}^{(t)}$ must not cancel exactly. 
It may be viewed as a step-wise, localized version of the \emph{path coefficients condition} of \citet{willem2020genetic}. 
Non-cancellation of directed-path contributions alone is not sufficient for faithfulness, since non-faithfulness can also arise from cancellation among treks, among other mechanisms \citep{willem2020genetic}.

\begin{remark}
By~\eqref{eq:neumann}, $\theta_{ij}^{(t)}$ is a sum of squarefree monomials in the nonzero entries of $B$, one for each directed path from $j$ to $i$ whose nodes after $j$ lie in $V_t$. 
Distinct directed paths have distinct edge sets and hence yield distinct monomials, so $\theta_{ij}^{(t)}$ is not the zero polynomial whenever $j\in\partial_i(U_t)$. 
Since a fixed DAG admits only finitely many prefix sets and finitely many pairs $(j,i)$, the union of the corresponding zero sets is Lebesgue-null. 
Hence Assumption~\ref{ass:non-cancellation} holds for Lebesgue-almost every choice of the nonzero edge coefficients.
\end{remark}

Combining the path-based inclusion in~\eqref{eq:support-inclusion} with Assumption~\ref{ass:non-cancellation} yields an exact characterization of the predictors needed for residual reconstruction. 
In particular, the active boundary coincides with the support of the corresponding regression coefficients.

\begin{proposition}[Characterization of the active boundary]
\label{prop:boundary-support}
Under Assumption~\ref{ass:non-cancellation}, for every prefix set $U_t$ and every remaining node $i\in V_t$,
\[
  \operatorname{supp}(\Theta^{(t)}_{i,\cdot})
  =\partial_i(U_t).
\]
Consequently,
\begin{equation}\label{eq:sparse-boundary-regression}
  x_i
  =\sum_{j\in\partial_i(U_t)}
    \theta_{ij}^{(t)}x_j+r_i^{(t)}.
\end{equation}
\end{proposition}
\begin{proof}
The inclusion~\eqref{eq:support-inclusion} gives one direction. 
Conversely, Assumption~\ref{ass:non-cancellation} gives $\theta_{ij}^{(t)}\neq0$ for every $j\in\partial_i(U_t)$, and hence $\partial_i(U_t)\subseteq\operatorname{supp}(\Theta^{(t)}_{i,\cdot})$. 
The support equality follows. 
Substituting it into~\eqref{eq:reduced-regression} gives the stated representation.
\end{proof}

Equation~\eqref{eq:sparse-boundary-regression} shows that regressing $x_i$ on $x_{\partial_i(U_t)}$ is sufficient to reconstruct the DirectLiNGAM residual $r_i^{(t)}$. 
The following proposition shows that this support is also minimal: it is the unique smallest support attaining the minimum risk.

\begin{proposition}[Minimality of the active boundary]
\label{prop:boundary-minimality}
Let $U_t$ be a prefix set, let $i\in V_t$, and suppose Assumption~\ref{ass:non-cancellation} holds. 
Then the least-squares coefficient of $x_i$ on $x_{U_t}$ is unique and equals $\Theta^{(t)}_{i,\cdot}$, whose support is $\partial_i(U_t)$.
Moreover, every subset $S\subseteq U_t$ whose best linear predictor attains the minimum squared-error risk satisfies $\partial_i(U_t)\subseteq S$.
Hence $\partial_i(U_t)$ is the unique smallest support attaining the minimum risk, and the corresponding residual is $r_i^{(t)}$.
\end{proposition}
\begin{proof}
Since
\[
x_{U_t}=(I-B_{U_tU_t})^{-1}e_{U_t}
\]
and the noise covariance is diagonal with strictly positive entries, $\Cov(x_{U_t})\succ0$. 
Hence the least-squares coefficient of $x_i$ on $x_{U_t}$ is unique. 
By Lemma~\ref{lem:reduced-regression}, it equals $\Theta^{(t)}_{i,\cdot}$ and has residual $r_i^{(t)}$; by Proposition~\ref{prop:boundary-support}, its support is $\partial_i(U_t)$.

Now let $S\subseteq U_t$ and suppose that its best linear predictor attains the same minimum risk. 
Extending its coefficient vector by zeros to $U_t$ gives a minimizer of the full regression on $x_{U_t}$. 
By uniqueness, this coefficient vector must equal $\Theta^{(t)}_{i,\cdot}$.
Therefore,
\[
\partial_i(U_t)
=\operatorname{supp}(\Theta^{(t)}_{i,\cdot})
\subseteq S.
\]
Since $\partial_i(U_t)$ itself attains the minimum risk by~\eqref{eq:sparse-boundary-regression}, it is the unique smallest support attaining that risk, with residual $r_i^{(t)}$.
\end{proof}

The residuals also retain the LiNGAM structure on the remaining variables. Indeed, by~\eqref{eq:reduced-regression},
\[
 r_{V_t}^{(t)}
 =B_{V_tV_t}r_{V_t}^{(t)}+e_{V_t}.
\]
Thus, reconstructing the same residuals preserves the LiNGAM model to which the next DirectLiNGAM ordering step is applied.

\begin{corollary}[Preservation of the ordering step]
\label{cor:ordering-preservation}
Let $U_t$ be a prefix set. 
If, for every $i\in V_t$, the residual is reconstructed by the regression of $x_i$ on $x_{\partial_i(U_t)}$, then the resulting residuals coincide with the residuals produced by DirectLiNGAM after $t$ ordering steps.
Consequently, the next DirectLiNGAM root-selection criterion yields the same scores, and hence the same set of minimizers.
\end{corollary}

\begin{corollary}[Characterization for pruning]
\label{cor:pruning-parent-recovery}
Let $C_i$ satisfy $\Pa_i\subseteq C_i\subseteq\NonDesc_i$.
Then the least-squares coefficient of $x_i$ on $x_{C_i}$ is unique, has support $\Pa_i$, equals $b_{ij}$ for $j\in\Pa_i$, and has residual $e_i$. 
Moreover, every subset $S\subseteq C_i$ attaining the minimum squared-error risk satisfies $\Pa_i\subseteq S$.
Hence $\Pa_i$ is the unique smallest support attaining the minimum risk.
\end{corollary}
\begin{proof}
Since $C_i$ contains no descendants of $i$, the noise $e_i$ is independent of $x_{C_i}$. 
Thus the structural equation
\[
 x_i=\sum_{j\in\Pa_i}b_{ij}x_j+e_i
\]
gives the least-squares projection of $x_i$ onto $x_{C_i}$.
Because $\Cov(x_{C_i})\succ0$, its coefficient vector is unique and has support $\Pa_i$. 
The same uniqueness argument as in Proposition~\ref{prop:boundary-minimality} then implies that every risk-minimizing support $S\subseteq C_i$ contains $\Pa_i$.
\end{proof}

\subsection{Active-Boundary Sparsity and Sample Size}
\label{subsec:boundary-sparsity}

Section~\ref{subsec:theory} characterizes which variables in the prefix set are needed to reconstruct a DirectLiNGAM residual. 
This subsection studies the size of this target set and its relation to the sample size $n$.

Define the (maximum) active-boundary sparsity along $\pi$ by 
\begin{equation}\label{eq:boundary-sparsity}
 s_{\max}(\pi)
 :=\max_{0\le t<p}\max_{i\in V_t^{\pi}}
 |\partial_i(U_t^{\pi})|.
\end{equation}

By Proposition~\ref{prop:boundary-minimality}, the cardinality of the active boundary $|\partial_i(U_t^{\pi})|$ is the minimum support size required to reconstruct the DirectLiNGAM residual for target $i$ at the population level. 
Hence $s_{\max}(\pi)$ is the largest such support size encountered along $\pi$.

The contrast with cumulative residualization is already visible in the simplest possible graph.

\begin{example}[Directed chain]\label{ex:chain-boundary}
Let $G$ be the chain $1\to2\to\cdots\to p$ with $\pi(i)=i$. 
For $t\ge1$ and $i\in V_t$, a node $j\in U_t$ belongs to $\partial_i(U_t)$ only if the path from $j$ to $i$ enters $V_t$ immediately after leaving $j$. 
In the chain this occurs only for $j=t$.
Hence
\[
\partial_i(U_t)=\{t\}
\]
for every $i\in V_t$, and therefore
\[
s_{\max}(\pi)=1.
\]
Thus, at every noninitial ordering step, the minimum support required to reconstruct any remaining residual has size one, independently of $p$.
\end{example}

The active-boundary sparsity can, however, behave very differently from the maximum in-degree. 
We next consider the random DAG construction proposed by \citet{wang2020high}.

\begin{example}[Wang--Drton random DAG]\label{ex:wang-drton-boundary}
Consider the random DAG construction of \citet{wang2020high} with the fixed causal order $1,\ldots,p$ and in-degree bound $J$.
For each node $v=2,\ldots,p$, the edge $(v-1)\to v$ is always present, and the number of parents is sampled as
\[
 d_v\sim\operatorname{Unif}\{1,\ldots,\min(v-1,J)\}.
\]
When $d_v>1$, the remaining $d_v-1$ parents are selected uniformly without replacement from $\{1,\ldots,v-2\}$.

Consider a prefix set $U_t=\{1,\ldots,t\}$ and a remaining node $i>t$.
The mandatory edges form the chain
\[
 t\to t+1\to\cdots\to i,
\]
so $t\in\partial_i(U_t)$. 
For $t>J$ and $j<t$, the probability that $j$ is selected as an additional parent of a node $v$ is
\[
\frac{J-1}{2(v-2)}.
\]
Therefore,
\[
\mathbb{E}\!\left[|\partial_i(U_t)|\right]
=
1+(t-1)
\left[
1-
\prod_{v=t+1}^{i}
\left(
1-\frac{J-1}{2(v-2)}
\right)
\right].
\]
For the $J=3$ setting used in our experiments, this simplifies to
\[
\mathbb{E}\!\left[|\partial_i(U_t)|\right]
=
1+(t-1)\frac{i-t}{i-2}.
\]
Taking the final target $i=p$ and
\[
t=\left\lceil\frac{p+1}{2}\right\rceil,
\]
we obtain
\[
\mathbb{E}[s_{\max}(\pi)]
\ge
\mathbb{E}\!\left[|\partial_p(U_t)|\right]
=
1+
\frac{\lfloor (p-1)^2/4\rfloor}{p-2}
=
\frac{p}{4}+O(1).
\]
Thus, in the $J=3$ setting, although the maximum in-degree is bounded by three, the expected active-boundary sparsity grows linearly with $p$.
More generally, a bounded maximum in-degree does not by itself imply bounded active-boundary sparsity.
\end{example}

This example shows that the number of predictors required for the residual reconstruction can increase with $p$ even when the maximum in-degree is fixed.

We next examine whether residual reconstruction using the active boundary remains feasible in the finite-sample data matrix.
For $S\subseteq V$, let $P_S$ denote the orthogonal projector onto the subspace $\operatorname{span}(X_S)$ spanned by the columns of $X_S$. 
Along a causal order $\pi$, define the \emph{oracle active-boundary residual} by
\begin{equation}\label{eq:oracle-boundary-residual}
 \widetilde R_i^{(t)}
 :=
 X_i-
 P_{\partial_i(U_t^\pi)}X_i,
 \qquad
 t=0,\dots,p-2,\quad i\in V_t^\pi.
\end{equation}

Because the columns of $X$ are centered, at most $n-1$ of them can be linearly independent. 
Thus $|\partial_i(U_t^\pi)|\le n-2$ makes linear independence of
\[
 \{X_j:j\in\partial_i(U_t^\pi)\}\cup\{X_i\}
\]
possible, but does not by itself guarantee it. 
We therefore impose the following localized nondegeneracy condition.

\begin{assumption}[Sample nondegeneracy of active-boundary regressions]
\label{ass:sample-nondegeneracy}
For every ordering step $t=0,\ldots,p-2$ and every $i\in V_t^\pi$,
\[
 \operatorname{rank}
 X_{\partial_i(U_t^\pi)\cup\{i\}}
 =
 |\partial_i(U_t^\pi)|+1
\]
whenever $|\partial_i(U_t^\pi)|\le n-2$.
\end{assumption}

Assumption~\ref{ass:sample-nondegeneracy} excludes exact linear dependencies only among the columns involved in the oracle active-boundary regressions. 
The condition is local: it does not require the full data matrix $X$ to have maximal rank, nor does it impose linear independence on subsets that do not arise in these regressions.

\begin{proposition}[Finite-sample feasibility of oracle boundary reconstruction]\label{prop:oracle-boundary-feasibility}
Suppose the sample nondegeneracy condition of Assumption~\ref{ass:sample-nondegeneracy} holds along a causal order $\pi$. If
\[
s_{\max}(\pi)\le n-2,
\]
then for every ordering step $t=0,\ldots,p-2$ and every remaining node $i\in V_t^\pi$, the oracle active-boundary residual satisfies
\[
\widetilde R_i^{(t)}\neq\mathbf 0.
\]
\end{proposition}
\begin{proof}
By~\eqref{eq:boundary-sparsity}, $|\partial_i(U_t^\pi)|\le s_{\max}(\pi)\le n-2$.
Hence Assumption~\ref{ass:sample-nondegeneracy} implies that the columns of $X_{\partial_i(U_t^\pi)\cup\{i\}}$ are linearly independent. 
Therefore
\[
 X_i\notin
 \operatorname{span}\{X_j:j\in\partial_i(U_t^\pi)\},
\]
so~\eqref{eq:oracle-boundary-residual} is nonzero.
\end{proof}

Proposition~\ref{prop:oracle-boundary-feasibility} applies regardless of the relative sizes of $n$ and $p$. 
Its main relevance here is the small-sample regime $n<p$, where it provides a finite-sample contrast to Proposition~\ref{prop:vanishing-residuals}. 
Under cumulative residualization, DirectLiNGAM is necessarily subject to deterministic vanishing, whereas oracle active-boundary reconstruction remains nondegenerate under $s_{\max}(\pi)\le n-2$ and Assumption~\ref{ass:sample-nondegeneracy}.

The key difference is that DirectLiNGAM applies one additional projection at each ordering step, whereas oracle active-boundary reconstruction returns to the original data and uses only $|\partial_i(U_t^\pi)|\le s_{\max}(\pi)$ predictors. 
Thus the number of predictors used for residual reconstruction is controlled by the graph rather than by the ordering step $t$. 
In particular, when $n<p$, the conditions of Proposition~\ref{prop:oracle-boundary-feasibility} allow the oracle residuals to remain nonzero even at ordering steps $t\ge n-1$, where cumulative residualization cannot continue with nonzero projection directions. 
For example, in the directed chain of Example~\ref{ex:chain-boundary}, $s_{\max}(\pi)=1$ independently of $p$.

\begin{remark}[Scope of Proposition~\ref{prop:oracle-boundary-feasibility}]
\label{rem:scope-of-claim}
Proposition~\ref{prop:oracle-boundary-feasibility} guarantees only that the oracle active-boundary residuals remain nonzero. 
It does not guarantee recovery of the active boundary. 
\end{remark}

Since $s_{\max}(\pi)$ is a graph-dependent quantity, it is useful to relate it to simpler structural quantities. 
For any node $i$, consider the prefix set immediately preceding $i$, namely $U_{\pi(i)-1}$. 
At this point all ancestors of $i$ are contained in the prefix set, and
\[
 \partial_i(U_{\pi(i)-1})=\Pa_i.
\]
Indeed, each parent reaches $i$ directly, whereas any path from a non-parent ancestor contains an intermediate ancestor already in the prefix set.
Therefore,
\[
 s_{\max}(\pi)\ge\max_{i\in V}|\Pa_i|.
\]

For an upper bound, we introduce a set that depends only on the prefix set and contains the active boundaries of all remaining nodes.

\begin{definition}[Outgoing frontier]\label{def:outgoing-frontier}
For a subset $U\subseteq V$, define its \emph{outgoing frontier} by
\[
 F(U)
 :=
 \bigl\{j\in U:\Ch_j\cap(V\setminus U)\neq\varnothing\bigr\}.
\]
\end{definition}

Thus $F(U)$ consists of the nodes in $U$ having at least one child outside $U$. 
Unlike the active boundary, it does not depend on a target node. 
Along a causal order, it is exactly the union of the active boundaries of the remaining nodes.

\begin{lemma}[Outgoing frontier as the union of active boundaries]\label{lem:frontier-union}
For every causal order $\pi$ and every step $t=0,\ldots,p-1$,
\[
 \bigcup_{i\in V_t^{\pi}}\partial_i(U_t^{\pi})=F(U_t^{\pi}).
\]
In particular, $\partial_i(U_t^{\pi})\subseteq F(U_t^{\pi})$, and hence $|\partial_i(U_t^{\pi})|\le|F(U_t^{\pi})|$, for every $i\in V_t^{\pi}$.
\end{lemma}
\begin{proof}
Let $j\in\partial_i(U_t^\pi)$ for some $i\in V_t^\pi$. 
By Definition~\ref{def:active-boundary}, there is a directed path $j=v_0\to v_1\to\cdots\to v_m=i$ with $v_1,\ldots,v_m\in V_t^\pi$. 
Hence $v_1\in\Ch_j\cap V_t^\pi$, so $j\in F(U_t^\pi)$. 
Conversely, let $j\in F(U_t^\pi)$. Then some $k\in\Ch_j\cap V_t^\pi$ exists, and the single-edge path $j\to k$ shows that $j\in\partial_k(U_t^\pi)$.
\end{proof}

Taking maxima in the inclusion $\partial_i(U_t^\pi)\subseteq F(U_t^\pi)$ and combining it with the lower bound above gives
\begin{equation}\label{eq:boundary-inout-bounds}
 \max_{i\in V}|\Pa_i|
 \;\le\;
 s_{\max}(\pi)
 \;\le\;
 \max_{0\le t<p}|F(U_t^\pi)|.
\end{equation}

These bounds give simple graph-level conditions for the sufficient condition $s_{\max}(\pi)\le n-2$. 
If some node has in-degree at least $n-1$, then the condition cannot hold for any causal order. 
Conversely, if $|F(U_t^\pi)|\le n-2$ for every $t$, then $s_{\max}(\pi)\le n-2$, and Proposition~\ref{prop:oracle-boundary-feasibility} applies. 
An in-degree of at least $n-1$ therefore rules out this sufficient condition, but does not by itself imply that an oracle residual must vanish.
The following examples illustrate these two cases.
\begin{example}[Hub]\label{ex:hub-boundary}
Let $G$ have edges $j\to p$ for $j=1,\dots,p-1$, so that node $p$ has in-degree $p-1$ and all other nodes are roots. 
For any causal order, $\pi(p)=p$, and
\[
 \partial_p(U_{p-1}^{\pi})=\Pa_p=\{1,\dots,p-1\}.
\]
Hence
\[
 s_{\max}(\pi)=p-1
\]
for every causal order. 
Thus the active-boundary sparsity grows with $p$ and cannot be reduced by reordering. 
In particular, $s_{\max}(\pi)\le n-2$ fails throughout the regime $n<p$.
Proposition~\ref{prop:oracle-boundary-feasibility} is therefore inapplicable in this example because its sufficient condition is not satisfied. 
This does not imply, however, that a particular oracle residual must vanish.
\end{example}

\begin{example}[Layered graph of bounded width]\label{ex:layered-boundary}
Partition $V$ into layers $L_1,\dots,L_m$ with $|L_k|\le w$ for all $k$, with the convention $L_0=\varnothing$, and suppose every edge joins consecutive layers, that is, $j\to i$ implies $j\in L_k$ and $i\in L_{k+1}$ for some $k$. 
Let $\pi$ order the layers consecutively. 
Fix a step $t<p$ and let $L_k$ be the first layer not fully contained in $U_t$. 
Then
\begin{equation*}
 U_t=L_1\cup\cdots\cup L_{k-1}\cup A,
 \qquad A\subsetneq L_k,
\end{equation*}
so $|A|\le w-1$. 
A node in $L_{k'}$ with $k'\le k-2$ has all of its children in $L_{k'+1}\subseteq U_t$ and hence does not belong to the outgoing frontier of Definition~\ref{def:outgoing-frontier}, so $F(U_t)\subseteq A\cup L_{k-1}$ and $|F(U_t)|\le(w-1)+w=2w-1$. 
By~\eqref{eq:boundary-inout-bounds},
\begin{equation*}
 s_{\max}(\pi)\le 2w-1,
\end{equation*}
independently of the number of layers and hence of $p$. 
Consequently, if $2w-1\le n-2$ and Assumption~\ref{ass:sample-nondegeneracy} holds, Proposition~\ref{prop:oracle-boundary-feasibility} applies independently of the number of layers and hence of $p$.
The directed chain of Example~\ref{ex:chain-boundary} is a special case with $w=1$, for which the bound is exact. 
Thus the graph may be arbitrarily deep, and $p$ arbitrarily large relative to $n$, provided its layer width remains bounded relative to $n$.
\end{example}

Taken together, these results show that finite-sample feasibility of oracle residual reconstruction can be controlled by the active-boundary sparsity relative to the sample size, rather than by the total number of variables.
Under Assumption~\ref{ass:sample-nondegeneracy},
\[
 s_{\max}(\pi)\le n-2
\]
is sufficient for all oracle active-boundary residuals used during causal ordering to remain nonzero. 
As the hub and bounded-width examples illustrate, this condition reflects structural complexity along the causal order and can either grow with $p$ or remain bounded independently of $p$.

\subsection{Algorithm and Implementation}
\label{subsec:algo}

Section~\ref{subsec:theory} characterizes the predictor sets required for residual reconstruction, while Section~\ref{subsec:boundary-sparsity} analyzes their size and finite-sample feasibility. 
In particular, the sparsity bound $s_{\max}(\pi)\le n-2$ gives a sufficient condition for the corresponding oracle residuals to remain nonzero under Assumption~\ref{ass:sample-nondegeneracy}. 
In observed data, however, the required predictor sets are unknown, so AdaPS-LiNGAM adaptively estimates them using best-subset selection.

\subsubsection{Adaptive best subset selection}
For a response vector $y\in\R^n$, a candidate index set $C\subseteq V$, the associated design matrix
\[
 W_C:=(W_j)_{j\in C}\in\R^{n\times|C|},
\]
and a prescribed cardinality $k\in\{0,\dots,|C|\}$, best subset selection (BSS) solves
\begin{equation}\label{eq:sample-bss}
 \widehat\beta(k)
 \in \arg\min_{\beta\in\R^{|C|}}
 \frac{1}{n}\Bigl\|y-\sum_{j\in C}\beta_j W_j\Bigr\|_2^2
 \quad\text{subject to}\quad
 \|\beta\|_0\le k,
\end{equation}
where $\beta=(\beta_j)_{j\in C}$ and $\|\beta\|_0:=|\{j\in C:\beta_j\neq0\}|$.


Applying BSS in practice raises two issues. 
First, solving~\eqref{eq:sample-bss} is a combinatorial optimization problem and can be computationally demanding when the candidate set is large. 
Second, the appropriate cardinality $k$ is generally unknown and needs to be specified or selected in some way. 
We refer to BSS together with selection of the cardinality $k$ as adaptive BSS.

To solve adaptive BSS, we use \texttt{abess}~\citep{zhu2020abess}, which uses a splicing-based algorithm for best-subset selection and selects the subset size using an information criterion. 
We write $\widehat\beta$ for the resulting coefficient vector and
\[
 \widehat S:=\{j\in C:\widehat\beta_j\neq0\}\subseteq C
\]
for its support.

\subsubsection{Use in causal ordering and pruning}
Let $X=(X_1,\dots,X_p)\in\R^{n\times p}$ denote the centered data matrix. 
We use the same sample iteration notation as in Section~\ref{subsec:directlingam}: $\widehat U_t$ contains the variables selected in the first $t$ positions and $\widehat V_t:=V\setminus\widehat U_t$ contains the variables that remain to be ordered. 
At ordering step $t$, AdaPS-LiNGAM returns to the original data and, for every $i\in\widehat V_t$, uses $X_{\widehat U_t}$ as the candidate design in adaptive BSS. 
The selected subset is used to construct a sample residual vector $\widehat R_i^{(t)}$, and the next root is chosen by minimizing $T_n(j;\widehat R^{(t)})$ over the residual data matrix $\widehat R^{(t)}:=(\widehat R_i^{(t)})_{i\in\widehat V_t}$. 
Proposition~\ref{prop:boundary-support} shows that the DirectLiNGAM residual can be reconstructed using only a subset of the prefix set, while Corollary~\ref{cor:ordering-preservation} shows that this reconstruction preserves the next DirectLiNGAM root-selection step. 
Further, the sparsity analysis in Section~\ref{subsec:boundary-sparsity} shows when such a reconstruction can remain nondegenerate in finite samples. 
Accordingly, adaptive BSS is used to retain only a subset of $\widehat U_t$, rather than residualizing against all of $\widehat U_t$.

After an estimated causal order $\widehat\pi:V\to V$ has been obtained, define the predecessor set of variable $i$ by
\[
 \widehat C_i:=\{j\in V:\widehat\pi(j)<\widehat\pi(i)\}.
\]
AdaPS-LiNGAM applies the same adaptive BSS procedure to pruning using $X_{\widehat C_i}$ as the candidate design. 
If the estimated order is consistent with the causal structure, $\Pa_i\subseteq \widehat C_i\subseteq\NonDesc_i$, and Corollary~\ref{cor:pruning-parent-recovery} shows that the corresponding regression coefficient has support $\Pa_i$.

Algorithm~\ref{alg:adaps} summarizes the full sample-level procedure. 
Here $\operatorname{AdaBSS}(y,W_C)$ denotes the adaptive BSS, 
returning a selected support $\widehat S\subseteq C$ and the corresponding coefficient vector $\widehat\beta\in\R^{|C|}$.

\begin{algorithm}[t]
\caption{AdaPS-LiNGAM}\label{alg:adaps}
\begin{algorithmic}[1]
\Require Centered data matrix $X\in\R^{n\times p}$
\State $\widehat U_0\gets\varnothing$ and $\widehat V_0\gets V$
\For{$t=0,\dots,p-2$}
  \For{each $i\in\widehat V_t$}
    \If{$\widehat U_t=\varnothing$}
      \State $\widehat R_i^{(t)}\gets X_i$
    \Else
      \State $(\widehat S_i^{(t)},\widehat\beta_i^{(t)})\gets\operatorname{AdaBSS}(X_i,X_{\widehat U_t})$
      \State $\widehat R_i^{(t)}\gets X_i-\sum_{j\in\widehat S_i^{(t)}}\widehat\beta_{i,j}^{(t)}X_j$
    \EndIf
  \EndFor
  \State $\widehat R^{(t)}\gets (\widehat R_i^{(t)})_{i\in\widehat V_t}$
  \State choose $\widehat j_t\in\arg\min_{j\in\widehat V_t} T_n(j;\widehat R^{(t)})$
  \State $\widehat\pi(\widehat j_t)\gets t+1$
  \State $\widehat U_{t+1}\gets \widehat U_t\cup\{\widehat j_t\}$ and $\widehat V_{t+1}\gets \widehat V_t\setminus\{\widehat j_t\}$
\EndFor
\State assign $\widehat\pi(j)\gets p$ for the unique $j\in\widehat V_{p-1}$
\State $\widehat B\gets 0_{p\times p}$
\For{each $i\in V$}
  \State $\widehat C_i\gets\{j\in V:\widehat\pi(j)<\widehat\pi(i)\}$
  \If{$\widehat C_i\neq\varnothing$}
    \State $(\widehat S_i,\widehat\beta_i)\gets\operatorname{AdaBSS}(X_i,X_{\widehat C_i})$
    \State set $\widehat b_{ij}\gets\widehat\beta_{i,j}$ for $j\in\widehat S_i$
  \EndIf
\EndFor
\State \Return $\widehat\pi,\widehat B$
\end{algorithmic}
\end{algorithm}

Comparing Algorithms~\ref{alg:directlingam} and~\ref{alg:adaps}, the sample root-selection criterion is identical. 
The difference during causal ordering is how the residual data matrices supplied to that criterion are constructed: DirectLiNGAM carries $R_i^{(t)}$ forward through repeated one-variable updates, whereas AdaPS-LiNGAM constructs $\widehat R_i^{(t)}$ directly from the original $X_i$ and a selected subset of $X_{\widehat U_t}$ at every step.

\section{Experiments}\label{sec:experiments}
We evaluate AdaPS-LiNGAM through two complementary experiments. 
The first examines how the methods behave as the number of variables increases while the sample size is fixed, with particular emphasis on the $n<p$ setting that motivates this work. 
The second fixes the number of variables and varies the sample size to assess how gracefully performance deteriorates as fewer observations become available.

\subsection{Experimental Setup}\label{subsec:setup-exp}
\subsubsection*{Data generation}
Synthetic data were generated from random LiNGAM models on four graph families.

For the Erd\H{o}s--R\'enyi DAGs, we first sampled a random topological order and then added each forward edge independently with probability $k/(p-1)$, where $k\in\{2,4,8\}$ controls the target average degree. 
We refer to these graph families as ER2, ER4, and ER8.

In addition, we considered the random DAG construction of \citet{wang2020high}. 
For each node $v=2,\ldots,p$, the number of parents $d_v$ was sampled uniformly from
\[
d_v \sim \operatorname{Unif}\left\{1,\ldots,\min(v-1,J)\right\},
\]
where $J$ is the specified in-degree bound. 
The edge $(v-1,v)$ was always included, and, when $d_v>1$, the remaining $d_v-1$ parents were selected uniformly without replacement from $\{1,\ldots,v-2\}$. 
Thus, the resulting DAG had maximum in-degree at most $J$ and was consistent with the fixed ordering. 
In our experiments, we set $J=3$ and refer to this graph family as \texttt{RG-indeg3}.

For all four graph families, the graph structure was generated independently of the subsequent SEM parameter generation. 
Nonzero edge coefficients were sampled with independent random signs and magnitudes drawn uniformly from $[0.5,2.0]$. 
The exogenous noises were mutually independent and either uniform or Laplace distributed. 
In the uniform setting, each noise term followed $\mathrm{Unif}(-\sqrt{3},\sqrt{3})$, giving unit variance, and in the Laplace setting, each noise term was centered Laplace with scale parameter $1$. 
After sampling, each observed variable was standardized, and the coefficient matrix was rescaled accordingly so that it remained consistent with the standardized samples.

\subsubsection*{Methods}
We compared our proposed method, AdaPS-LiNGAM, with two existing methods: DirectLiNGAM \citep{shimizu2011directlingam} and HighDimDirectLiNGAM \citep{wang2020high}. 

DirectLiNGAM used the pairwise likelihood score of \citet{hyvarinen2013pairwise}, with its ordering step implemented using the Python package \texttt{lingam}.
For final edge estimation, we replaced the package's default pruning routine with an adaptive-Lasso implementation based on \texttt{Ridge} and \texttt{LassoLarsIC} from \texttt{scikit-learn}.
Each variable was regressed on its predecessors in the estimated order, and adaptive weights were constructed from a ridge pilot estimate.
\texttt{LassoLarsIC} was then run with the BIC criterion, and the selected coefficients were refit by ridge regression on the selected support.
The same implementation was used in all reported settings to keep the DirectLiNGAM baseline consistent.
When $n\ge p$, \texttt{LassoLarsIC} was used without an explicit noise-variance specification.
When $n<p$, the same procedure was used except that the noise variance supplied to \texttt{LassoLarsIC} was fixed at $1$ to satisfy the software requirements for underdetermined regressions.

HighDimDirectLiNGAM is a LiNGAM-based method designed for high-dimensional settings and exploits sparsity through an assumed bound on the maximum in-degree \citep{wang2020high}. We included it as a baseline because it provides an existing approach specifically intended for causal discovery when the number of variables is large relative to the sample size.
For implementation, we used the \texttt{HighDimDirectLiNGAM} class provided by the Python package \texttt{lingam}. The package default parameters $J=3$, $K=4$, and $\alpha=0.5$ were adopted, where $J$ is the assumed largest in-degree, $K$ is the degree of the moment used to capture non-Gaussianity, and $\alpha$ controls pruning of false parents. 
We fixed $J=3$ across all experiments rather than tuning it to the true maximum in-degree of each generated graph, since the latter is generally unknown in practical applications and would constitute oracle information in our simulations. This common setting also keeps the computational cost manageable, as larger values of $J$ substantially increase the cost of the method. To ensure that the evaluation also includes a setting consistent with this assumed in-degree bound, we included \texttt{RG-indeg3}, whose maximum in-degree is bounded by three. In contrast, the maximum in-degree of the ER graph families can exceed three, particularly for denser graphs.

We implemented AdaPS-LiNGAM in Python. Residual reconstruction during causal ordering and final pruning used \texttt{abess} (version 0.4.11) \citep{zhu2022abess}, specifically \texttt{abess.linear.LinearRegression}, which implements the ABESS splicing algorithm of \citet{zhu2020abess}. We set \texttt{max\_parents="auto"}, so the candidate support sizes were determined by the package's default sequential search. When the candidate set has size $d$, this search explores sparse support sizes up to order $n/(\log\log n\,\log d)$. For support-size selection, we considered two information criteria: the generalized information criterion (GIC), corresponding to the SIC-type criterion discussed in \citet{zhu2020abess}, and the extended Bayesian information criterion (EBIC) \citep{chen2008ebic}. These choices define the GIC and EBIC variants of AdaPS-LiNGAM reported below.

For root selection, we used the same pairwise likelihood criterion as in DirectLiNGAM (\texttt{measure="pwling"}). No prior knowledge was supplied in any experiment.

\subsubsection*{Evaluation metrics}
Let $A,\widehat A\in\{0,1\}^{p\times p}$ denote the binarized true and estimated adjacency matrices in child-parent orientation, so that $A_{ij}=1$ if and only if $x_j$ is a parent of $x_i$. We evaluated directed edge recovery by precision, recall, and F1 score:
\begin{align*}
 \operatorname{Precision}(A,\widehat A)
 &:=\frac{\sum_{i,j} A_{ij}\widehat A_{ij}}{\sum_{i,j}\widehat A_{ij}},\\
 \operatorname{Recall}(A,\widehat A)
 &:=\frac{\sum_{i,j} A_{ij}\widehat A_{ij}}{\sum_{i,j} A_{ij}},\\
 \operatorname{F1}(A,\widehat A)
 &:=\frac{2\,\operatorname{Precision}(A,\widehat A)\operatorname{Recall}(A,\widehat A)}{\operatorname{Precision}(A,\widehat A)+\operatorname{Recall}(A,\widehat A)}.
\end{align*}
Precision quantifies how many estimated directed edges are correct, recall quantifies how many true directed edges are recovered, and F1 score summarizes their balance. 
To evaluate the estimated causal order $\widehat\pi$, we used the Topological Order Divergence (TOD) \citep{rolland2022score}:
\begin{equation*}
 \operatorname{TOD}(A,\widehat\pi)
 :=\sum_{i\neq j} A_{ij}\,\mathbf 1\!\left\{\widehat\pi(i)<\widehat\pi(j)\right\},
\end{equation*}
which counts true parent-child relations whose child is placed before its parent in the estimated order. 
Hence smaller TOD indicates a more reliable causal ordering. 
For graph-level discrepancy, we used the directed Structural Hamming Distance (SHD) \citep{tsamardinos2006maxmin}:
\begin{equation*}
 \operatorname{SHD}(A,\widehat A)
 :=\sum_{i,j} A_{ij}(1-\widehat A_{ij})
 +\sum_{i,j}(1-A_{ij})\widehat A_{ij}
 -\sum_{i,j} A_{ij}\widehat A_{ji}.
\end{equation*}
This quantity combines missing edges, extra edges, and reversed edges into a single directed graph error measure, so smaller SHD is better. 
We also recorded the wall-clock runtime of one call to the fitting routine.

\subsubsection*{Experimental environment}
All experiments were performed on a Linux server running Ubuntu 24.04.3 LTS, using Python (v3.10.19), with figures prepared using Matplotlib (v3.10.9). Each experiment was repeated over $100$ random seeds and parallelized over $50$ workers, while the internal numerical libraries were restricted to single-thread execution. The reported experiments were carried out on a server with $56$ physical cores ($112$ logical cores) and $755$~GiB RAM.

\subsection{Small-sample Settings}\label{subsec:small-sample}
We fix the sample size at $n=100$ and increase the number of variables from $p=10$ to $200$, thereby moving from $n\ge p$ to $n<p$. 
The four graph families, ER2, ER4, ER8, and \texttt{RG-indeg3}, are considered under both uniform and Laplace noise. 
The results are shown in Figures~\ref{fig:ex1-uniform} and~\ref{fig:ex1-laplace}.

As the dimensionality approaches and then exceeds the sample size, AdaPS-LiNGAM (EBIC) maintains a clear advantage over DirectLiNGAM in both causal order and edge recovery. 
The difference is particularly pronounced at $p=200$: under uniform noise, the F1 score of AdaPS-LiNGAM (EBIC) is $0.905$, $0.785$, and $0.557$ for ER2, ER4, and ER8, respectively, compared with $0.315$, $0.319$, and $0.263$ for DirectLiNGAM. 
The corresponding TOD values are also substantially smaller for AdaPS-LiNGAM. 
On ER2, for instance, TOD is $11.14$ versus $77.00$, accompanied by an SHD of $34.39$ versus $430.20$. 
The consistent improvement in both measures suggests that the benefit of the proposed residual construction extends beyond the final edge selection and affects the estimated causal ordering itself.

Graph density nevertheless remains an important source of difficulty.
Performance declines from ER2 to ER8 for every method, reflecting the greater complexity of denser graphs. 
AdaPS-LiNGAM (EBIC) is comparatively more robust to this increase in density: its F1 score at $p=200$ under uniform noise falls from $0.905$ on ER2 to $0.557$ on ER8, while DirectLiNGAM remains low, changing only from $0.315$ to $0.263$. 

The \texttt{RG-indeg3} graphs provide a useful complementary case. 
Although their maximum in-degree is bounded by three, they are considerably harder than the ER2 and ER4 graphs; at $p=200$, AdaPS-LiNGAM (EBIC) achieves an F1 score of $0.445$. As shown in Example~\ref{ex:wang-drton-boundary}, however, the Wang--Drton construction can have active-boundary sparsity that grows linearly with $p$ even when the maximum in-degree is bounded by three. 
This provides a structural explanation for why maximum in-degree alone does not capture the difficulty of the residual reconstruction problem. 
Additional analyses of active-boundary sparsity are provided in Appendix~\ref{app:active-boundary}.

The choice of information criterion mainly changes the balance between precision and recall. GIC tends to favor larger supports, yielding relatively high recall at the expense of additional false positives, whereas EBIC produces a sparser solution with a better overall precision--recall balance. 
For example, on ER2 at $p=200$ under uniform noise, GIC gives precision $0.401$ and recall $0.857$, compared with $0.887$ and $0.923$ for EBIC. 
The resulting F1 scores are $0.546$ and $0.905$, respectively.

HighDimDirectLiNGAM exhibits substantially weaker edge-recovery performance in this experiment, despite its much lower computational cost. 
The high-dimensional procedure of \citet{wang2020high} was developed for settings in which the number of variables is large relative to the sample size, but their high-dimensional simulation uses $n=3p/4$. 
In our experiment, the ratio $n/p$ decreases from one to $0.5$ as $p$ increases to $200$, making the problem increasingly sample-limited. 
Together with the fixed choice $J=3$, which may not cover the maximum in-degree of the ER graphs, this provides a more demanding setting for HighDimDirectLiNGAM.

The resulting edge-recovery performance is particularly poor at large $p$. 
Under uniform noise on ER2 at $p=200$, for example, its precision is only $0.035$ and its recall is $0.687$, yielding an F1 score of $0.066$. 
Its SHD also grows rapidly with dimensionality; on ER8 at $p=200$, the underlying SHD exceeds $2000$. 
These results indicate that the computationally efficient high-dimensional procedure does not maintain accurate edge recovery under the increasingly sample-limited settings considered here.

The same overall ordering of the methods is observed under Laplace noise, although the task is generally more difficult. 
For example, at $p=200$ on ER2, the F1 score of AdaPS-LiNGAM (EBIC) is $0.679$ under Laplace noise, compared with $0.905$ under uniform noise. 
The advantage of EBIC-based AdaPS-LiNGAM over DirectLiNGAM nonetheless remains substantial, indicating that the main qualitative finding is not specific to the uniform-noise setting.

The computational comparison is more nuanced. 
At $p=200$, AdaPS-LiNGAM requires roughly $576$--$579$ seconds per run, compared with about $685$ seconds for DirectLiNGAM, while HighDimDirectLiNGAM takes only about $40$ seconds. 
Thus, under the fixed-$n$ setting considered here, AdaPS-LiNGAM combines substantially better statistical performance with a runtime comparable to DirectLiNGAM, although it does not match the speed of HighDimDirectLiNGAM.

\begin{figure}[!htbp]
  \centering
  \includegraphics[width=\textwidth]{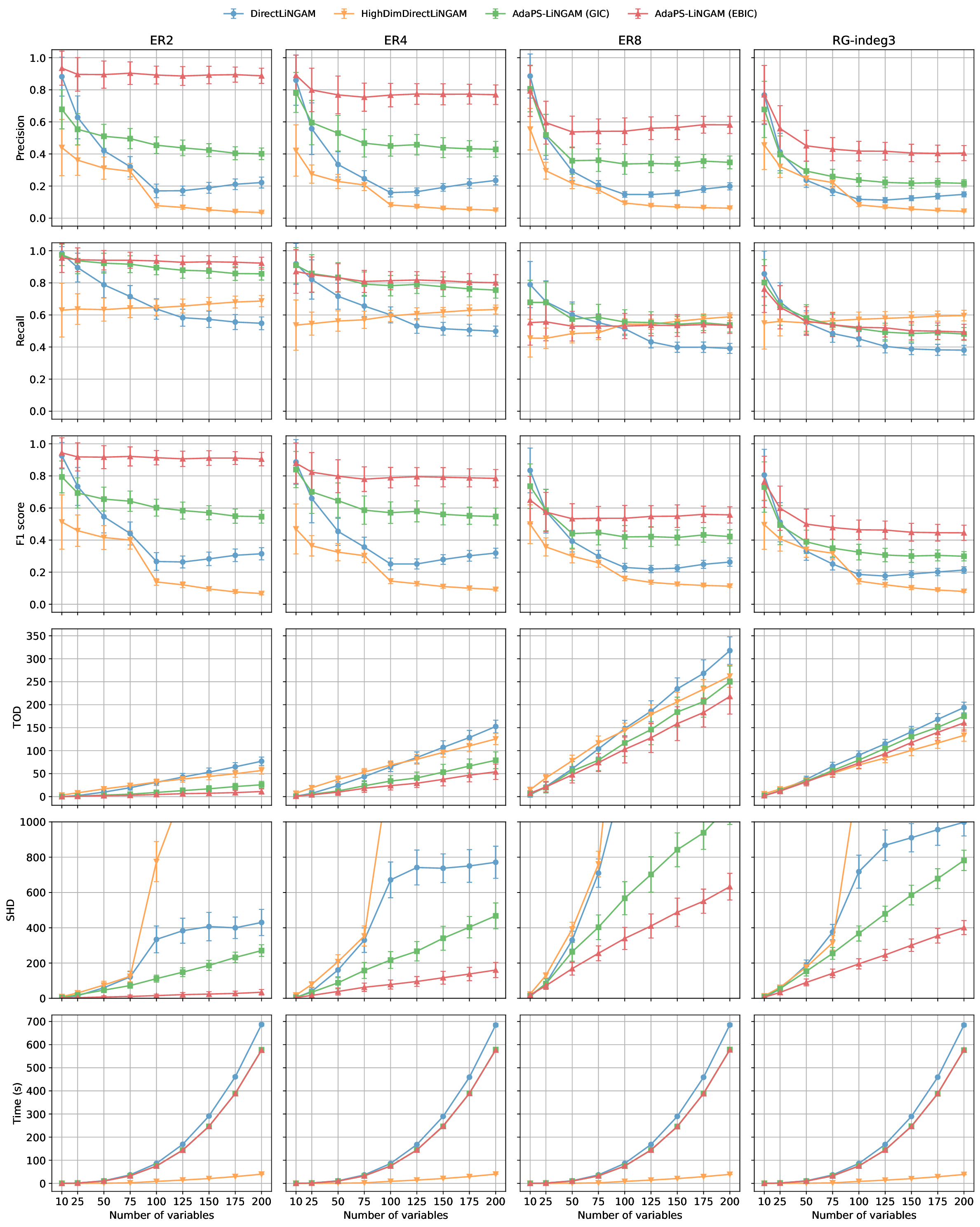}
  \caption{Results for the small-sample experiment with $n=100$ under uniform noise, with $p\in\{10,25,50,75,100,125,150,175,200\}$. Markers show the mean over $100$ trials and error bars $\pm1$ standard deviation. The SHD axis is truncated at $1000$.}
  \label{fig:ex1-uniform}
\end{figure}

\begin{figure}[!htbp]
  \centering
  \includegraphics[width=\textwidth]{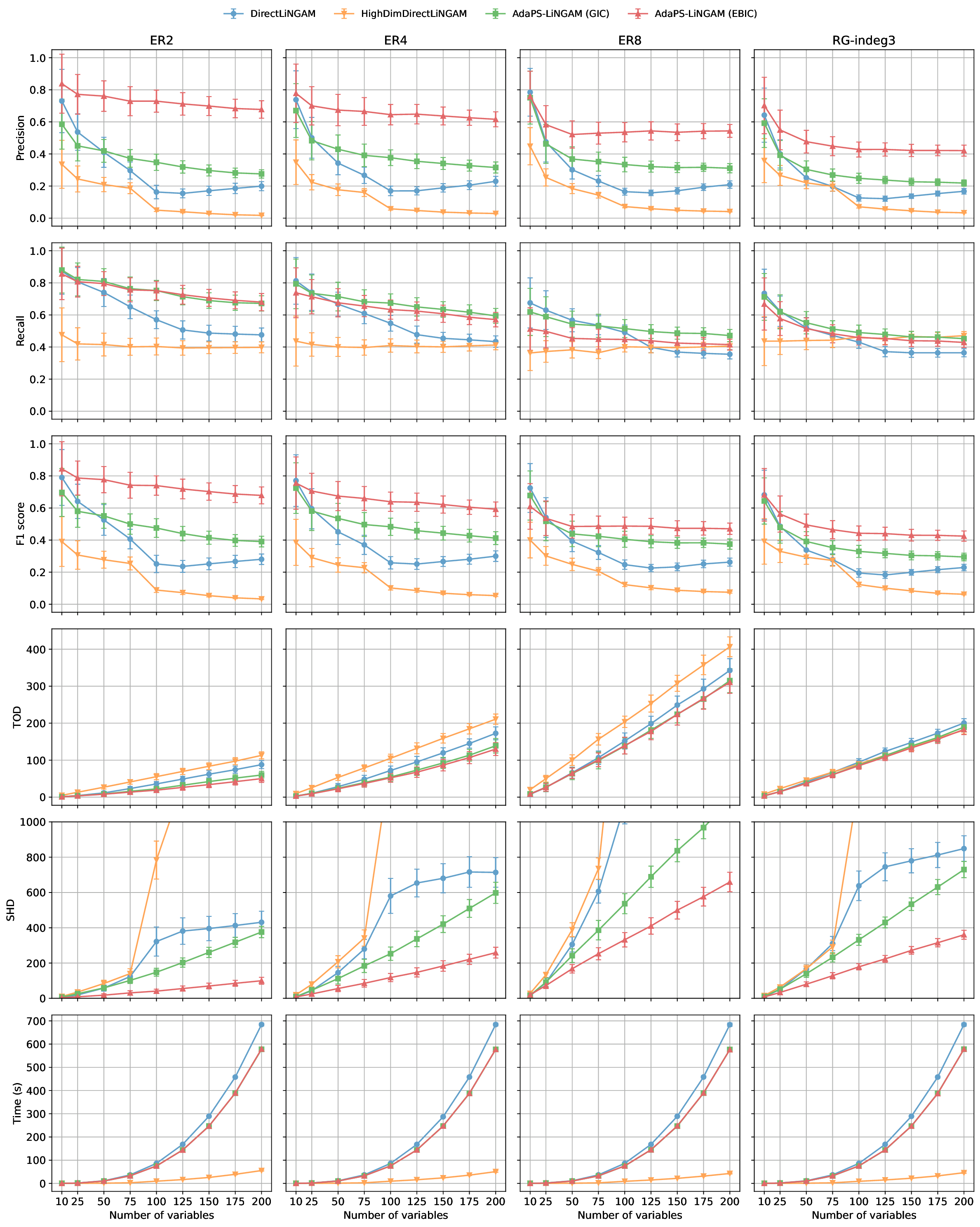}
  \caption{Results for the small-sample experiment with $n=100$ under Laplace noise, with $p\in\{10,25,50,75,100,125,150,175,200\}$. Markers show the mean over $100$ trials and error bars $\pm1$ standard deviation. The SHD axis is truncated at $1000$.}
  \label{fig:ex1-laplace}
\end{figure}

\subsection{Performance under Reduced Sample Sizes}\label{subsec:scalability}

Here we fix $p=100$ and vary the sample size from $n=100$ to $2000$.
Rather than characterizing a formal sample-complexity threshold, this experiment asks how rapidly each method loses accuracy as observations become scarce. 
The four graph families are again evaluated under both uniform and Laplace noise, with the results shown in Figures~\ref{fig:ex2-uniform} and~\ref{fig:ex2-laplace}.

AdaPS-LiNGAM (EBIC) shows a markedly slower deterioration than DirectLiNGAM as the sample size is reduced. 
Under uniform noise on ER2, for example, its F1 score changes only from $0.997$ at $n=2000$ to $0.914$ at $n=100$, whereas DirectLiNGAM falls from $0.978$ to $0.267$. 
In terms of the absolute change, this corresponds to a decrease of only $0.083$ for AdaPS-LiNGAM (EBIC), compared with $0.711$ for DirectLiNGAM. 
The same pattern persists as the graphs become denser or structurally different. 
On ER8, the corresponding F1 scores are $0.982$ and $0.536$ for AdaPS-LiNGAM (EBIC), versus $0.993$ and $0.229$ for DirectLiNGAM at $n=2000$ and $n=100$, respectively.
For \texttt{RG-indeg3}, the decrease is from $0.983$ to $0.464$ for AdaPS-LiNGAM (EBIC), compared with a drop from $0.990$ to $0.186$ for DirectLiNGAM.

Graph error measures show the same difference in sensitivity to sample reduction. 
On uniform-noise ER2, SHD rises from $0.67$ to $15.49$ for AdaPS-LiNGAM (EBIC) as $n$ falls from $2000$ to $100$, whereas DirectLiNGAM rises from $4.49$ to $334.05$. 
The contrast is even more pronounced for ER8: AdaPS-LiNGAM (EBIC) changes from $15.02$ to $339.57$, while DirectLiNGAM increases from $5.69$ to $1281.72$.

The behavior of the two information criteria remains consistent with the first experiment. 
GIC tends to preserve recall by accepting more edges, while EBIC places greater emphasis on controlling false positives. 
On uniform-noise ER2, for instance, GIC reaches a recall of essentially one once $n\ge500$, but its SHD remains $24.42$ even at $n=2000$, compared with $0.67$ for the EBIC variant. 
Thus, the difference between the two criteria becomes primarily a question of sparsity rather than the ability to identify existing edges.

HighDimDirectLiNGAM shows a distinct pattern as the sample size changes. 
Its relative advantage is more apparent in causal-order recovery when $n$ is close to $p$, although this does not translate into comparable edge-recovery accuracy. 
This is particularly clear for \texttt{RG-indeg3}: at $n=100$ ($n/p=1$), HighDimDirectLiNGAM achieves the lowest TOD, $69.01$, compared with $73.13$ for AdaPS-LiNGAM (EBIC) and $90.22$ for DirectLiNGAM. 
In contrast, its F1 score at the same sample size is only $0.144$, with an SHD of $1289.40$. 
This separation between ordering and edge recovery is consistent with the analysis of \citet{wang2020high}, where conservative pruning can preserve a valid ordering while leaving non-parental ancestors in the estimated edge set.

As the sample size increases, the ordering advantage does not lead to a corresponding improvement in edge recovery. 
For \texttt{RG-indeg3}, TOD decreases from $69.01$ at $n=100$ to $55.34$ at $n=2000$, whereas F1 score increases only from $0.144$ to $0.343$. 
The effect is even weaker for ER8, where F1 score changes from $0.160$ to only $0.211$ over the same range. 
Thus, within the sample sizes considered here, additional observations improve the ordering only modestly and do not remove the large gap in edge-recovery accuracy.

The computational cost of AdaPS-LiNGAM increases substantially with the number of observations. 
On uniform-noise ER2, its average runtime rises from approximately $75$ seconds at $n=100$ to $792$ seconds at $n=2000$, whereas DirectLiNGAM increases from approximately $87$ to $129$ seconds. 
This reflects the increasing cost of adaptive BSS, including both the regression fits and the range of candidate support sizes considered by \texttt{abess}. 
Consequently, the statistical benefit of AdaPS-LiNGAM is most pronounced when observations are limited, while the computational advantage shifts toward DirectLiNGAM as the sample size becomes large.

\begin{figure}[!htbp]
  \centering
  \includegraphics[width=\textwidth]{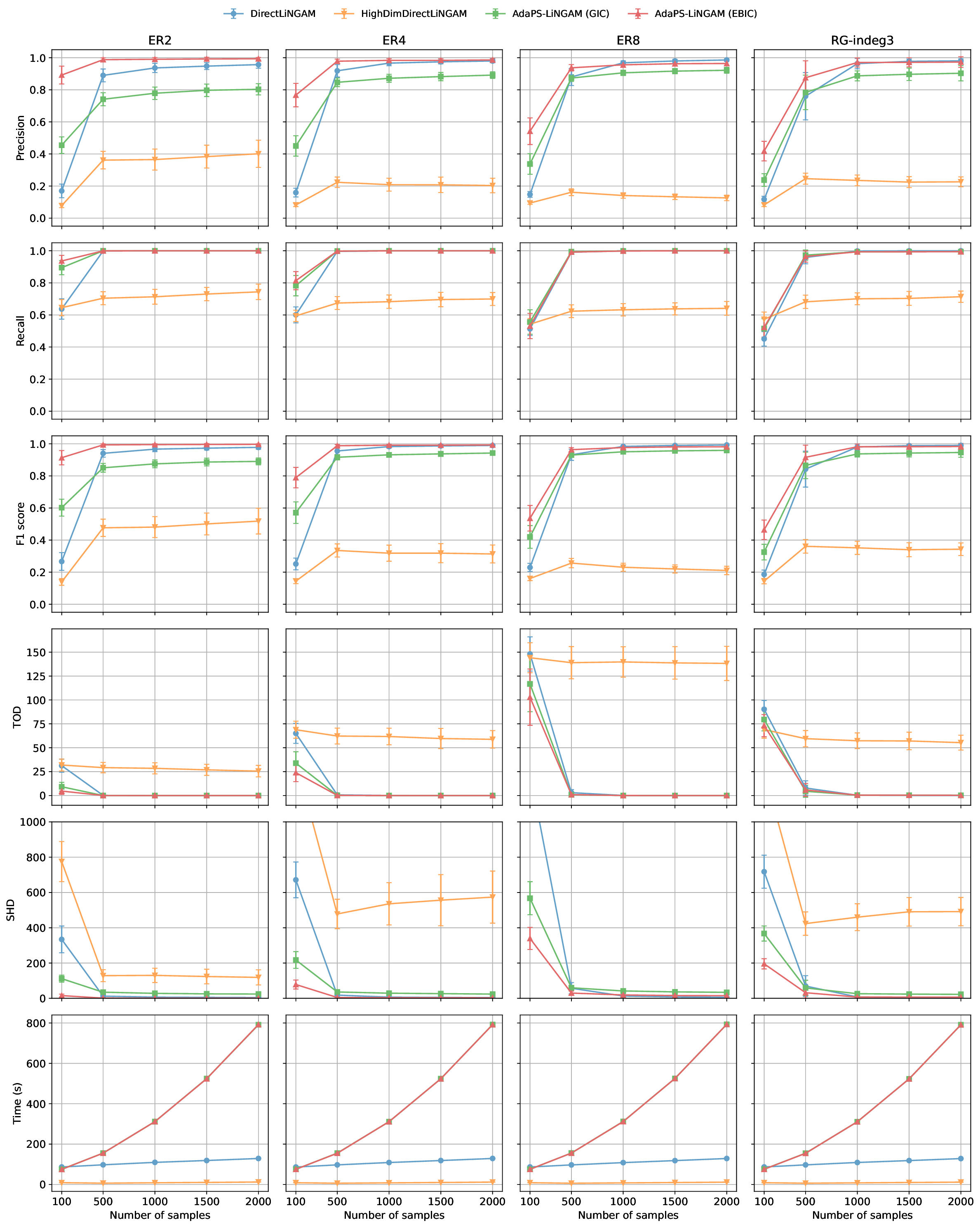}
  \caption{Results for the reduced-sample experiment with $p=100$ under uniform noise, with $n\in\{100,500,1000,1500,2000\}$. Markers show the mean over $100$ trials and error bars $\pm1$ standard deviation. The SHD axis is truncated at $1000$.}
  \label{fig:ex2-uniform}
\end{figure}

\begin{figure}[!htbp]
  \centering
  \includegraphics[width=\textwidth]{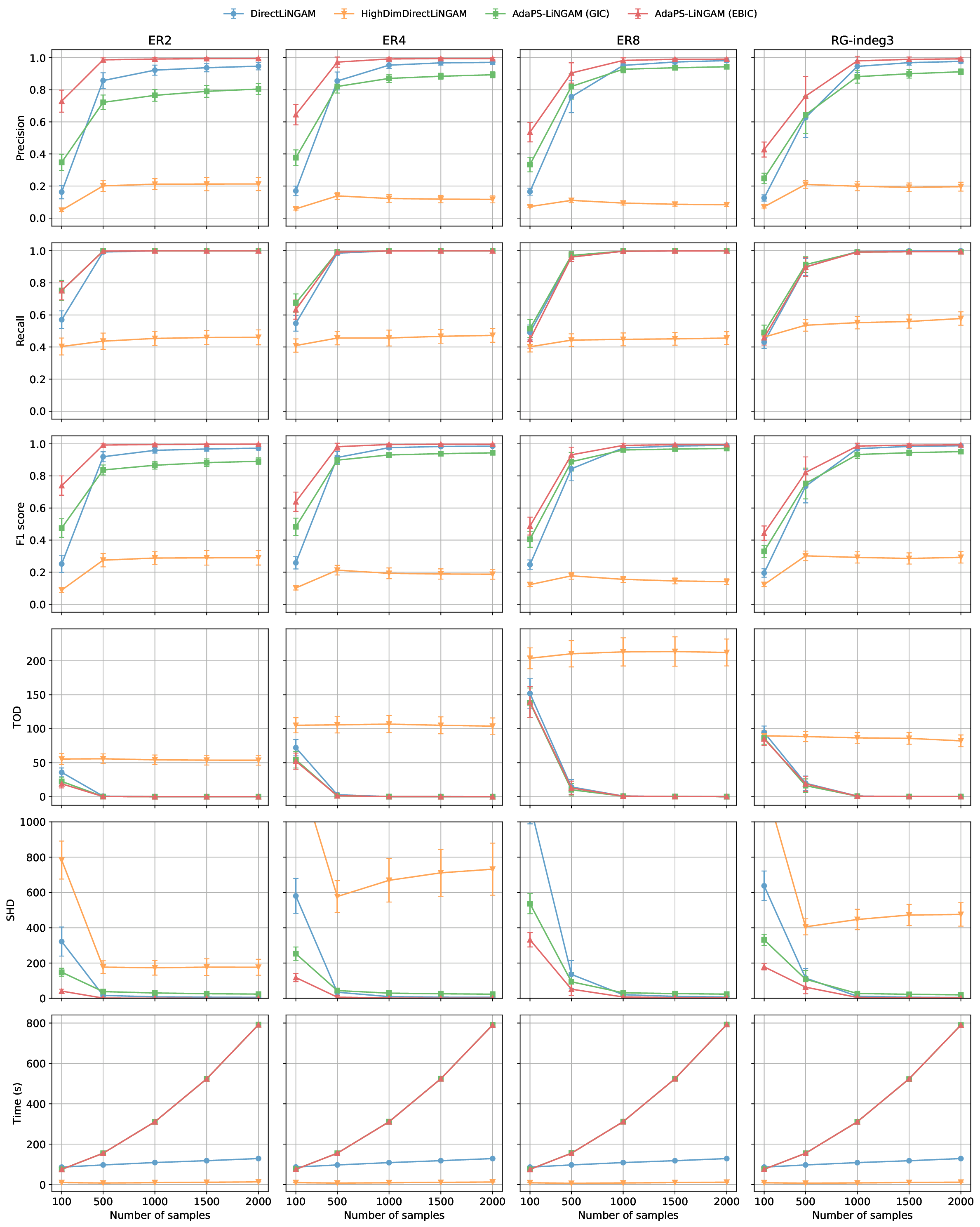}
  \caption{Results for the reduced-sample experiment with $p=100$ under Laplace noise, with $n\in\{100,500,1000,1500,2000\}$. Markers show the mean over $100$ trials and error bars $\pm1$ standard deviation. The SHD axis is truncated at $1000$.}
  \label{fig:ex2-laplace}
\end{figure}

In summary, we have confirmed that AdaPS-LiNGAM remains effective as the number of observations becomes limited relative to the number of variables. 
In particular, the EBIC-based variant generally achieved more accurate causal-order and edge recovery than DirectLiNGAM and showed a more gradual deterioration as the sample size decreased. 
The GIC-based variant tended to retain high recall but selected denser graphs, resulting in more false positives than the EBIC-based variant. 
HighDimDirectLiNGAM was substantially faster and showed an advantage in causal-order recovery in some settings. 
Its edge-recovery results should be interpreted in light of the assumed maximum in-degree $J=3$, which matches \texttt{RG-indeg3} but can be smaller than the true maximum in-degree of the ER graphs. 
Even in \texttt{RG-indeg3}, however, the ordering advantage did not translate into comparable edge-recovery accuracy. 
Overall, these results support the effectiveness of adaptive sparse residual reconstruction in sample-limited settings across different graph structures and noise distributions.

\section{Conclusion and Future Work}\label{sec:conclusion}

In this paper, we proposed AdaPS-LiNGAM, which extends the DirectLiNGAM framework by reconstructing residuals from a subset of previously selected variables rather than through cumulative residualization. 
We focused on sample-limited settings in which the number of observations is constrained and the number of variables is comparable to or larger than the sample size. 
We characterized the active boundary as the unique minimal predictor set required for exact residual reconstruction under a localized non-cancellation condition and showed that residual reconstruction using this set yields the same residuals as cumulative DirectLiNGAM residualization and therefore preserves the next root-selection step. 
We further showed that, under a local sample nondegeneracy condition, sufficiently sparse active boundaries prevent the oracle residuals from becoming degenerate even when the number of variables exceeds the sample size. 
Empirically, AdaPS-LiNGAM, particularly with EBIC-based support selection, achieved more accurate causal-order and graph recovery than DirectLiNGAM in these sample-limited settings and degraded more gradually as the sample size was reduced.

Several directions remain for future work. 
First, finite-sample statistical guarantees for AdaPS-LiNGAM should be established, including conditions under which the active boundary is reliably recovered and such recovery leads to consistent causal-order and graph estimation. 
Second, the small-sample behavior of the root-selection criterion inherited from DirectLiNGAM should be analyzed separately to clarify its reliability under limited observations. 
Third, the current characterization of active-boundary sparsity and the sufficient condition relating its size to the sample size could be developed into more practical and readily verifiable assumptions. 
More generally, it would be valuable to obtain explicit guarantees on how small the sample size can become while still permitting reliable causal discovery with AdaPS-LiNGAM.


\backmatter












\section*{Statements and Declarations}
\subsection*{Conflict of Interest}
On behalf of all authors, the corresponding author states that there is no conflict of interest.






\begin{appendices}

\section{Additional Analysis of Active-Boundary Sparsity}
\label{app:active-boundary}

\subsection{RG-indeg3 Graphs}
\label{subsec:app-rg-indeg3}

The \texttt{RG-indeg3} graphs used in our experiments are generated according to the Wang--Drton construction with in-degree bound $J=3$ and fixed causal order $\pi(i)=i$ \citep{wang2020high}. 
Thus, $1,\ldots,p$ is a true causal order of each generated DAG. As shown in Example~\ref{ex:wang-drton-boundary}, this construction can have active-boundary sparsity that grows linearly with $p$, despite its bounded maximum in-degree.

Figure~\ref{fig:app-rg-indeg3-boundary} compares the empirical mean of $s_{\max}(\pi)$ with the theoretical lower bound on $\mathbb{E}[s_{\max}(\pi)]$,
\[
1+\frac{\lfloor (p-1)^2/4\rfloor}{p-2}.
\]
The empirical means follow the same approximately linear trend and remain above the theoretical lower bound over the range considered.

\begin{figure}[htbp]
  \centering
  \includegraphics[width=0.75\textwidth]{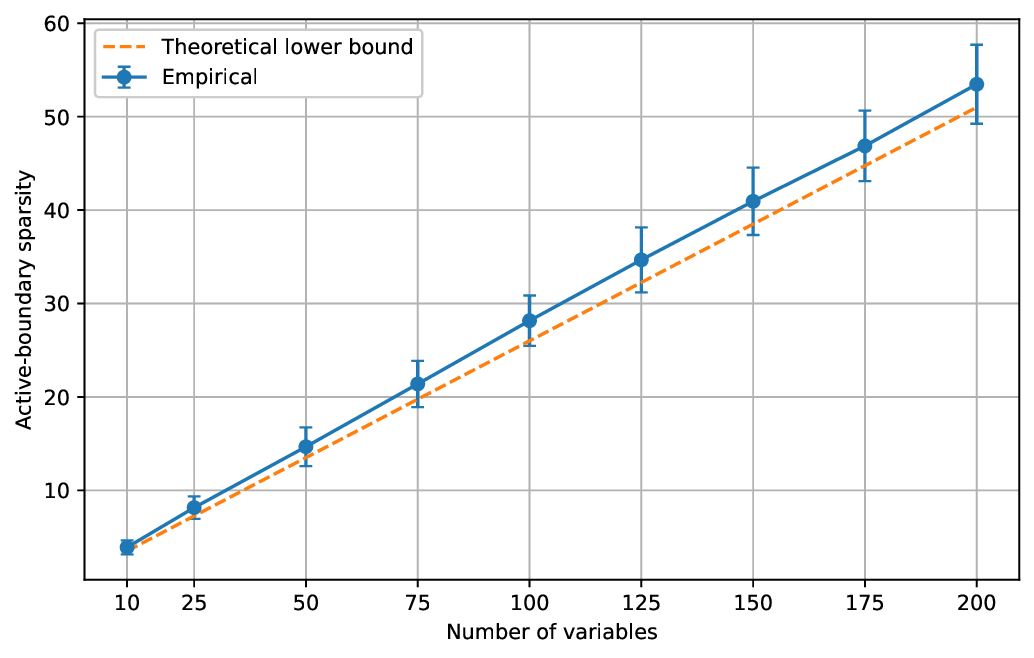}
  \caption{Empirical mean active-boundary sparsity of the \texttt{RG-indeg3} graphs under the fixed true causal order $\pi(i)=i$, together with the theoretical lower bound on $\mathbb{E}[s_{\max}(\pi)]$ from Example~\ref{ex:wang-drton-boundary}. Error bars indicate one standard deviation across graph realizations.}
  \label{fig:app-rg-indeg3-boundary}
\end{figure}

\subsection{Comparison Across Graph Families}
\label{subsec:app-boundary-comparison}

Next, we compared the active-boundary sparsity across the four graph families used in Section~\ref{sec:experiments}. 
Because $s_{\max}(\pi)$ depends on the causal order, we further examined its variation across valid orders. 
For each true graph realization, we sampled ten valid causal orders using randomized topological sorting. 
For RG-indeg3, the mandatory chain makes the causal order unique, so all sampled orders coincide with $\pi(i)=i$.
Figure~\ref{fig:app-boundary-comparison} shows the median across graph realizations and sampled orders, with the shaded regions indicating the interquartile range.

Active-boundary sparsity increases with graph density: ER2 has the smallest values, followed by ER4 and ER8. 
The \texttt{RG-indeg3} graphs also show a substantial increase with $p$, despite their maximum in-degree being bounded by three. 
Thus, maximum in-degree alone does not characterize the size of the active boundary.

\begin{figure}[htbp]
  \centering
  \includegraphics[width=0.8\textwidth]{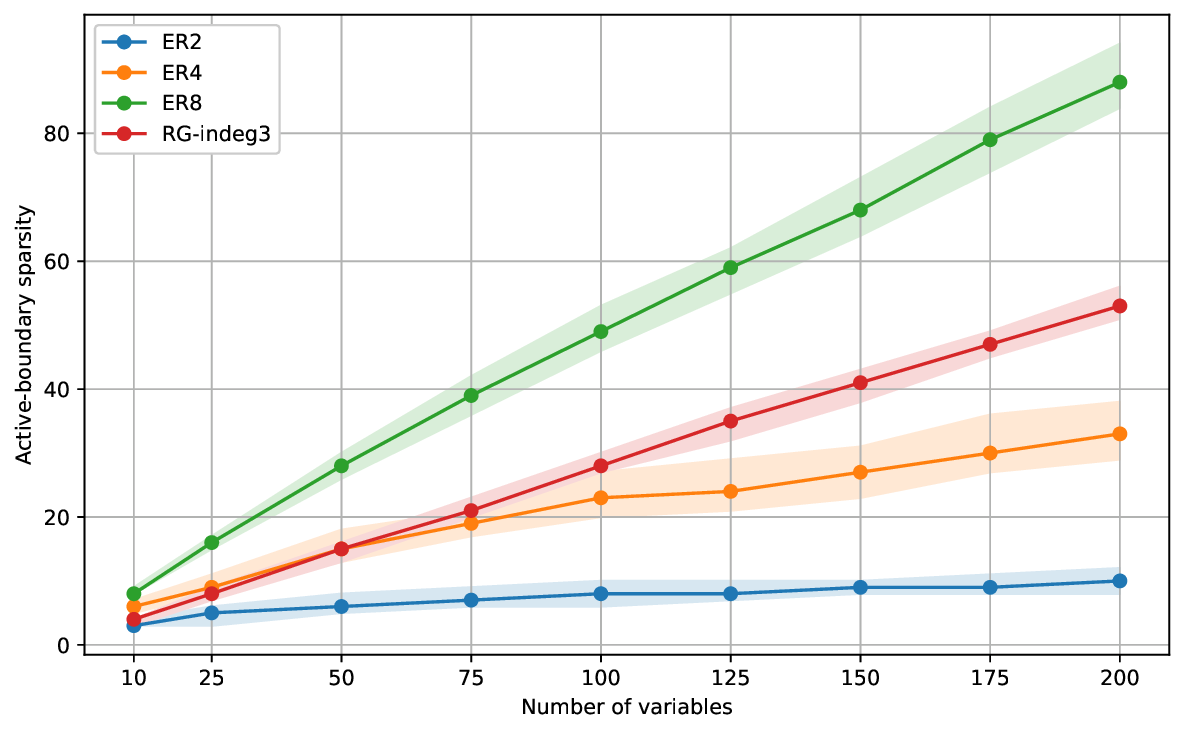}
  \caption{Active-boundary sparsity across the four graph families used in
  Section~\ref{sec:experiments}. For each true graph realization, ten valid causal orders
  were sampled. Curves show the median across graph realizations and sampled
  orders, and shaded regions indicate the interquartile range.}
  \label{fig:app-boundary-comparison}
\end{figure}

\subsection{Relation to Experimental Results}
\label{subsec:app-boundary-results}

The active-boundary results provide a structural interpretation of the performance differences observed in Section~\ref{subsec:small-sample}. 
In particular, \texttt{RG-indeg3} has a substantially larger active-boundary burden than would be suggested by its maximum in-degree alone, which is consistent with its relatively difficult edge-recovery performance at large $p$.

More generally, the ordering of active-boundary sparsity across graph families is qualitatively consistent with the experimental results: the denser ER8 graphs have larger active-boundary sparsity than ER2 and also exhibit greater degradation in edge recovery. 
These comparisons suggest that active-boundary sparsity captures a component of residual reconstruction difficulty that is not reflected by maximum in-degree alone.

These results are intended as a structural interpretation of the empirical findings, not as a finite-sample guarantee for AdaPS-LiNGAM. 
The theoretical result in Section~\ref{subsec:boundary-sparsity} guarantees nonvanishing oracle residuals under $s_{\max}(\pi)\le n-2$ and the stated sample nondegeneracy condition, but does not guarantee finite-sample support recovery or correct causal-order recovery.
\end{appendices}


\bibliography{sn-bibliography}

\end{document}